\documentclass{article}

\usepackage[utf8]{inputenc}
\usepackage{amsfonts}
\usepackage{amsmath}
\usepackage{amsthm}
\usepackage{mathtools}
\usepackage{xcolor}
\usepackage{algorithm,algorithmic}
\usepackage{thm-restate}

\usepackage[accepted]{icml2023}

\newcommand{\lse}[1]{\log \sum_{#1} \exp}
\newcommand{\Expect}{\mathbb{E}}
\newcommand{\reals}{\mathbb{R}}
\newcommand{\Ac}{\mathcal{A}}
\newcommand{\Sc}{\mathcal{S}}

\newcommand{\Fc}{\mathcal{F}}

\newcommand{\argmax}{\mathop{\rm argmax}}
\newcommand{\eg}{\emph{e.g.}}
\newcommand{\ie}{\emph{i.e.}}
\newcommand{\etc}{\emph{etc.}}

\newcommand{\epregret}{\mathcal{R}}
\newcommand{\cumregret}{\mathcal{BR}}

\newcommand{\Error}{\mathrm{Optimism}}
\newcommand{\Dist}{\mathrm{Dist}}

\usepackage{xspace}
\DeclarePairedDelimiterX{\infdivx}[2]{(}{)}{%
  #1\;\delimsize|\delimsize|\;#2%
}
\newcommand{\kld}[2]{\ensuremath{\mathrm{KL}\infdivx{#1}{#2}}\xspace}

\icmltitlerunning{~\hfill Efficient Exploration via Epistemic-Risk-Seeking Policy Optimization \hfill \thepage}

\begin{document}

\twocolumn[
\icmltitle{Efficient Exploration via Epistemic-Risk-Seeking Policy Optimization}

\begin{icmlauthorlist}
\icmlauthor{Brendan O'Donoghue}{dm}
\end{icmlauthorlist}

\icmlaffiliation{dm}{Google DeepMind, London}

\icmlcorrespondingauthor{Brendan O'Donoghue}{bodonoghue85@gmail.com}

\icmlkeywords{Reinforcement learning, Exploration, Bayes, Policy gradient}

\vskip 0.3in
]

\printAffiliationsAndNotice{}  %

\begin{abstract}
    Exploration remains a key challenge in deep reinforcement learning (RL). Optimism in the face of uncertainty is a well-known heuristic with theoretical guarantees in the tabular setting, but how best to translate the principle to deep reinforcement learning, which involves online stochastic gradients and deep network function approximators, is not fully understood. In this paper we propose a new, differentiable optimistic objective  that when optimized yields a policy that provably explores efficiently, with guarantees even under function approximation. Our new objective is a zero-sum two-player game derived from endowing the agent with an epistemic-risk-seeking utility function, which converts uncertainty into value and encourages the agent to explore uncertain states. We show that the solution to this game minimizes an upper bound on the regret, with the `players' each attempting to minimize one component of a particular regret decomposition. We derive a new model-free algorithm which we call `epistemic-risk-seeking actor-critic' (ERSAC), which is simply an application of simultaneous stochastic gradient ascent-descent to the game. Finally, we discuss a recipe for incorporating off-policy data and show that combining the risk-seeking objective with replay data yields a double benefit in terms of statistical efficiency. We conclude with some results showing good performance of a deep RL agent using the technique on the challenging `DeepSea' environment, showing significant performance improvements even over other efficient exploration techniques, as well as improved performance on the Atari benchmark.

\end{abstract}

\section{Introduction}
\label{s-intro}
Reinforcement learning (RL) involves an agent interacting with an environment over time attempting to maximize its total return \cite{sutton:book, puterman2014markov, meyn2022control}. Initially the agent does not know about the environment and must learn about it from experience. As the agent navigates the environment it receives noisy observations which it can use to update its (posterior) beliefs about the environment \cite{ghavamzadeh2015bayesian}. Therefore, the RL problem is a \emph{statistical inference problem} wrapped in a \emph{control problem}, and the two problems must be tackled simultaneously for good data efficiency \cite{lu2021reinforcement}. This is because the policy of the agent affects the data it will collect, which in turn affects the policy, and so on. This is in contrast to supervised learning, where the performance of a classifier (for instance) does not influence the data it will later observe. Failure to properly consider the statistical aspect of the RL problem will result in agents that require exponential amounts of experience for good performance. So far, deep RL as a field has largely accepted this tradeoff, requiring enormous computational budgets to solve relatively simple problems. On the other hand, correctly considering the statistical inference problem and the control problem together has the potential to dramatically reduce the compute requirements to solve problems and potentially unlock new domains and capabilities far outside of the range of current agents. 

Understood in this way, RL is about choosing what actions to take, and consequently which data to collect, in order to maximize long-term return. To do this an agent must sometimes take actions that lead to states where it has epistemic uncertainty about the value of those states, and sometimes take actions that lead to more certain payoff. The tension between these two modes is the `explore-exploit' dilemma \cite{auer2002using, kearns2002near, dimitrakakis2018decision}. When it comes to exploration in \emph{deep} RL there are two main focus areas of research. The primary line of work is generating better \emph{estimates of uncertainty}, typically by exploiting some aspect of a neural network \cite{singh2004intrinsically, barto2013intrinsic, stadie2015incentivizing, bellemare2016unifying, ostrovski2017count, burda2018exploration, pathak2017curiosity}. Getting accurate uncertainty estimates from deep neural networks is a `holy grail' of research in deep learning in general \cite{osband2021evaluating}, and in reinforcement learning good uncertainty estimates are crucial for good performance of any practical exploration algorithm. 
The second area of research is \emph{how best to use uncertainty estimates for efficient exploration}, which is the focus of this work and as such any of the referenced methods for generating uncertainty estimates are compatible with the approach discussed herein. A lot of prior work in this area has simply converted the uncertainty estimates into an `optimism in the face of uncertainty' bonus added to the rewards and then applied off-the-shelf RL algorithms to the modified Markov decision process (MDP) \cite{dayan1996exploration, strehl2008analysis, bellemare2016unifying, tang2017exploration}. This approach is inspired by theoretical results based on optimism bonuses which show that in an episodic tabular MDP setting where the modified MDP is solved exactly, %
these strategies can yield good regret bounds \cite{auer2008near, jaksch2010near, azar2017minimax, jin2018q}. 
However, translating the performance to deep RL has been challenging. Consider the fact that some of the most impressive results in modern deep RL have had no sophisticated exploration strategies, relying instead on simple local dithering strategies \cite{mnih-dqn-2015, silver2016mastering, berner2019dota} or making extensive use of human data \cite{vinyals2019grandmaster}. 

Although optimism is the most popular exploration technique in deep RL, there are several alternative approaches. One line of research is not to consider uncertainty explicitly, but instead to add some structured noise to dithering, such as L{\'e}vy flights \cite{dabney2020temporally}, or adding noise to the weights of the neural network \cite{fortunato2017noisy, plappert2017parameter}. These approaches have shown some promising results although they fall strictly into the category of heuristic and do not achieve good performance on challenging unit-test exploration domains like DeepSea \cite{osband2019behaviour}. Another line of research involves Thompson sampling and various approximations to it \cite{thompson1933likelihood, strens2000bayesian, osband2013more, russo2018tutorial, osband2016deep}. Although Thompson sampling has excellent performance in tabular settings it is not yet clear how to translate that performance into deep RL settings reliably as a full implementation of Thompson sampling requires sampling from the posterior over policies, which is intractable for all but the simplest tabular domains. Another drawback of Thompson sampling is that it cannot handle either the multi-agent case nor the constrained case \cite{o2020stochastic, o2021vbos}.
Since we expect real-world agents to be in situations with multiple agents and to be bound by constraints this is a major disadvantage.

In this paper we endow a policy-gradient agent with an epistemic-risk-seeking utility function which summarizes both the expected return and the epistemic uncertainty into a single value \cite{o2021klearning, eriksson2019epistemic}. How risk-seeking the agent is is controlled by a single scalar parameter which is tuned (\ie, learned) to balance exploration and exploitation. The approach is based on a \emph{dual} view of the recent `K-learning' algorithm, which is a value learning, model-based, Bayesian RL approach with a guaranteed Bayesian regret bound in tabular domains \cite{o2021klearning}. We derive a model-free and policy-based algorithm, which allows us to approximately solve for the optimal policy using stochastic feedback and online experience using policy gradients, and to use a deep neural network to parameterize our policy. Moreover, we can show that the approach enjoys Bayesian regret guarantees even in the face of function approximation. The final algorithm we present is an extension of policy gradient \cite{sutton1999policy, konda2003onactor} with entropy regularization. Combining policy gradient with entropy regularization is a common heuristic and typically a small amount of entropy `bonus' is used to discourage the policy from becoming deterministic and thereby losing the ability to `explore' \cite{mnih2016asynchronous}. That being said, simply adding entropy regularization is not sufficient for deep exploration since entropy regularization only encourages local dithering \cite{osband2016thesis, o2018uncertainty}. In this work we show that entropy regularization combined with a carefully tuned uncertainty bonus is a principled approach to deep exploration. Our approach formulates the problem as a two-player game where one player is attempting to find the policy that maximizes the optimistic reward and the other player is tuning how risk-seeking the policy is in order to minimize expected regret. The solution of this game yields the optimal K-learning policy with the associated performance guarantees. Unlike standard optimism approaches the K-learning policy is stationary (\ie, not dependent on the number of elapsed episodes other than through the posteriors) and stochastic, and it varies slowly as data is collected, which makes it more amenable to online approximation.  Unlike Thompson sampling, K-learning does not require a full sample from the posterior at each episode and it can handle both the multi-agent and the constrained cases when suitably modified \cite{o2020stochastic, o2021vbos}. In practice on a hard exploration unit-test our approach outperforms deep RL approximations to both Thompson sampling and optimism, as we shall show in the numerical experiments. Our results suggest that the approach in this manuscript may close some of the gap between theory and practice for efficient exploration in deep RL.

\section{Preliminaries}
\label{s-preliminaries}
We consider an RL problem where an agent interacts with an unknown environment over a number of episodes. We model the environment as a finite state-action time-inhomogeneous MDP given by the tuple $\mathcal{M} = (\Sc, \Ac, L, \{P_l\}_{l=1}^L, \{R_l\}_{l=1}^L, \rho)$, where $L$ is the horizon length, $\Sc = \Sc_1 \cup \cdots \cup \Sc_L \cup \{s_{L+1}\}$ is the set of states including terminating state $s_{L+1}$, $\Ac$ is the set of possible actions, $P_l: \Sc_l \times \Ac \rightarrow \Delta(\Sc_{l+1})$ denotes the state transition kernel at layer $l$, $R_l: \Sc_l \times \Ac \rightarrow \Delta(\reals)$ is the reward function at layer $l$ with mean reward $r_l \in \reals^{|\Sc_l|\times|\Ac|}$, and $\rho \in \Delta(\Sc_1)$ is the initial state distribution.  A policy $\pi \in \Delta(\Ac)^{|\Sc|}$ is a distribution over actions for each state, and we shall denote the probability of action $a$ in state $s$ at timestep $l$ as $\pi_l(s, a)$. The agent starts in some state $s_1 \in \Sc_1$ sampled according to $\rho$, then for each step in the episode $l=1,\ldots, L$ the agent is in state $s_l$, takes action $a_l \sim \pi_l(s_l, \cdot)$, receives reward sampled from $R_l(s_l, a_l)$, and transitions to state $s_{l+1} \in \Sc_{l+1}$ according to $P_l(\cdot \mid s_l, a_l)$. The episode ends when the terminating state $s_{L+1}$ is reached, the initial state is sampled again and another episode begins. 

For a given policy $\pi$ we define value functions for each $(s, a) \in \Sc_l \times \Ac$, $l=1, \ldots, L$, as
\begin{align*}
Q_l^\pi(s,a) &= r_l(s,a) + \sum_{s^\prime \in \Sc_{l+1}} P_l(s^\prime \mid s,a) V_{l+1}^\pi(s^\prime), \\
V_l^\pi(s) &= \sum_a \pi_l(s,a) Q_l^\pi(s,a),
\end{align*}
where we define $V^\pi_{L+1} \equiv 0$.
The optimal values are defined for $l=1, \ldots, L$ as
\begin{align*}
\label{e-opt-qv}
Q_l^\star(s,a) &= r_l(s,a) + \sum_{s^\prime \in \Sc_{l+1}} P_l(s^\prime\mid s,a) V_{l+1}^\star(s^\prime),
\\
V_l^\star(s) &= \max_a Q_l^\star(s,a),
\end{align*}
and we define $V^\star_{L+1} \equiv 0$. The policy that achieves the max is given by
\[
\pi_l^\star(s, a) = \mathbf{1}(a = \argmax(Q_l^\star(s,a))), \quad l=1,\ldots, L,
\]
assuming the argmax is unique, otherwise any policy that has support only on the maximum entries of $Q^\star$ is optimal.

\subsection{Regret}
\label{s-regret}
The regret of a policy is the expected shortfall between the performance of the policy and the optimal performance. In this paper we take a Bayesian approach, which is to say we assume the agent has access to prior information about the MDP, represented by a distribution $\phi$, and we are interested in the expected regret with respect to this prior. Concretely, for a policy $\pi$ we define the Bayesian regret for a single episode as
\[
\epregret(\pi, \phi) = \Expect_{\mathcal{M} \sim \phi}(\Expect_{s \sim \rho} (V_1^{\mathcal{M},\star}(s) - V_1^{\mathcal{M},\pi}(s))).
\]
For clarity we have made the dependence of the value functions on $\mathcal{M}$ explicit here, but we shall suppress the dependency in the notation hereafter.
If algorithm $\mathrm{Alg}$ produces policy sequence $\pi^1, \pi^2, \ldots$ based on observed histories $\Fc_1, \Fc_2, \ldots$, where $\Fc_t$
is all the observed history of states, actions, and rewards before episode $t$ 
then, due to the tower property of conditional expectation, the \emph{cumulative Bayesian regret} of $\mathrm{Alg}$ over $N$ episodes is given by
\begin{equation}
\label{e-bayes-regret}
\cumregret(\mathrm{Alg}, \phi) =\Expect \sum_{t=1}^N \epregret(\pi^t, \phi^t) 
\end{equation}
where $\phi^t = \phi(\cdot \mid \Fc_t)$. Loosely speaking, agents that have low regret explore efficiently and generate high reward. So minimizing the cumulative Bayesian regret is important for good performance.

\section{K-Learning}
For the value functions in \S \ref{s-preliminaries} to be computable they require exact knowledge of the mean reward $r$ and transition matrix $P$. However, in reinforcement learning these are initially unknown and must be learned about from experience. K-learning was derived by endowing the agent with a \emph{risk-seeking} exponential utility function $u: \reals \rightarrow \reals$ which converts uncertainties to value, defined for any $\tau \geq 0$ as $u_\tau(x) = \tau (\exp(x / \tau) - 1)$.
We can compute the \emph{certainty-equivalent} value under this utility for any random variable $X: \Omega \rightarrow \reals$ as $J_\tau = u_\tau^{-1} (\Expect u_\tau(X)) = \tau \log \Expect \exp (X / \tau)$, and from Jensen's inequality we have $J_\tau \geq \Expect X$ for all $\tau \geq 0$. For example, random variable $X \sim \mathcal{N}(\mu, \sigma^2)$ has certainty equivalent value under $u_\tau$ of $J_\tau = \mu + \sigma^2 / 2 \tau$. Clearly greater uncertainty (or risk) $\sigma$ increases this value, and $\tau \geq 0$ controls the tradeoff. 
In the context of reinforcement learning the uncertainty we are interested in is the \emph{epistemic} uncertainty about the unknown parameters of the MDP, and the risk-seeking utility function can be used to summarize the beliefs about an unknown MDP into a risk-seeking value. As shown by \citet{o2021klearning} the risk-seeking values are computable by solving a Bellman equation. Concretely, given posterior information $\phi$ we define the `risk-seeking' reward function for each $(s, a) \in \Sc_l \times \Ac$, $l=1, \ldots, L$, as
\[
r_{l, \tau}(s,a) = \bar r_l(s,a) + \sigma_l^2(s,a)/2\tau,
\]
where $\bar r = \Expect_\phi r$ and $\sigma \in \reals^{|\Sc|\times|\Ac|}$ is a measure of the uncertainty about the MDP under $\phi$ (see 
\citet{o2021klearning} for details on how $\sigma$ should be chosen, for now we shall just assume it is given).
Then for any policy $\pi$ and constant $\tau > 0$ we define risk-seeking value functions for each $(s, a) \in \Sc_l \times \Ac$, $l=1, \ldots, L$, as
\[
K_{l, \tau}^\pi(s,a) = r_{l, \tau}(s,a) + \sum_{s^\prime \in \Sc_{l+1}} \bar P_l(s^\prime\mid s,a) J_{l+1, \tau}^\pi(s^\prime),
\]
\begin{equation}
\label{e-kpi}
J_{l, \tau}^\pi(s) = \sum_a \pi_l(s,a) K_{l, \tau}^\pi(s,a) + \tau H(\pi_l(s, \cdot)),
\end{equation}
where $\bar P = \Expect_\phi P$ and $H$ denotes the entropy \cite{cover2012elements}, and we define $J^\pi_{L+1, \tau} \equiv 0$.
Similarly to the optimal Q-values, we can define optimal K-values for each $(s, a) \in \Sc_l \times \Ac$, $l=1, \ldots, L$, and any $\tau > 0$ as follows
\[
K_{l, \tau}^\star(s,a) = r_{l, \tau}(s,a) + \sum_{s^\prime \in \Sc_{l+1}} \bar P_l(s^\prime\mid s,a) J_{l+1, \tau}^\star(s_l^\prime),
\]
where
\begin{align*}
\label{e-jstar}
J_{l, \tau}^\star(s) &=
\max_{\pi_l \in \Delta(\Ac)}\left( \sum_a \pi_l(s,a) K_{l, \tau}^\star(s,a) +  \tau H(\pi_l(s,\cdot))\right) \\&= \tau \lse{a \in \Ac} (K_{l, \tau}^\star(s, a) / \tau),
\end{align*}
where again we define $J_{L+1, \tau}^\star \equiv 0$.
The policy that achieves the max is given by the `Boltzmann' policy over the K-values, that is, for each $(s, a) \in \Sc_l \times \Ac$, $l=1, \ldots, L$
\begin{equation}
\label{e-k-policy}
\pi_{l, \tau}^\star(s, a) = \exp\left( \frac{K_{l, \tau}^\star(s,a) - J_{l, \tau}^\star(s)}{\tau} \right).
\end{equation}
Observe that if the agent has no uncertainty (\ie, $\sigma = 0$), then letting $\tau \rightarrow 0$ recovers the original $Q$ and $V$ formulations in \S \ref{s-preliminaries}.
The risk-seeking Bellman equation captures both the expected value and the uncertainty, and both propagate through the MDP to other states and actions. It is the `temperature' parameter $\tau$ that is controlling the trade-off between them. So far $\tau$ is a free-variable, in the sequel we shall show how to optimize it so as to minimize regret.

The main result of \citet{o2021klearning} is that following the policy in Eq.~\eqref{e-k-policy} guarantees a sublinear Bayesian regret bound for appropriate choices of $\sigma$ and $\tau$. In other words, the policy associated with the optimal K-values balances exploration and exploitation efficiently. However, finding the policy requires solving a Bellman equation for the optimal K-values and the analysis was restricted to tabular cases. This paper builds on that work in three main ways:
\vspace{-2mm}
\begin{enumerate}
\setlength\itemsep{0em}
    \item We present a new objective over policies, rather than values, which can be solved using policy gradients to obtain the policy in Eq.~\eqref{e-k-policy}.
    \item The algorithm we derive is entirely model-free, whereas the previous work was model-based.
    \item We extend the analysis and experiments to cover non-tabular and function approximation cases.
\end{enumerate}
\vspace{-2mm}
All the quantities we presented in this section are functions of the current beliefs $\phi$, however, for brevity we have suppressed this dependence in the notation.

\section{Saddle-Point Problem}
If we assume that the posterior over the reward and transition functions are layerwise-independent, then it is straightforward to show that for any $\tau \geq 0$ and for $l=1, \ldots, L$
\[
K_{l, \tau}^\pi(s,a) \geq \Expect_\phi Q_l^\pi(s,a), \quad 
J_{l, \tau}^\pi(s) \geq \Expect_\phi V_l^\pi(s).
\]
Furthermore, in \cite{o2021klearning} it was shown that under some additional assumptions the optimal values satisfy for $l=1, \ldots, L$
\[
K_{l, \tau}^\star(s,a) \geq \Expect_\phi Q_l^\star(s,a), \quad
J_{l, \tau}^\star(s) \geq \Expect_\phi V_l^\star(s)
\]
for an appropriate choice of $\sigma$ and any $\tau \geq 0$.
This means that the K-values are \emph{optimistic}, and following policy \eqref{e-k-policy} is an instance of optimism in the face of uncertainty. For our purposes in this paper we shall assume the following bound holds.
\begin{restatable}{ass}{ass1}
\label{ass-upper-bound}
$\Expect_{s \sim \rho} J^\star_{1, \tau}(s) \geq \Expect_{s \sim \rho} \Expect_\phi V_1^\star(s), \ \forall \ \tau \geq 0$.
\end{restatable}
Under Assumption~\ref{ass-upper-bound}, finding the \emph{tightest} bound in the family requires solving $\min_\tau \Expect_{s \sim \rho} J_{1, \tau}^\star(s)$, and since for any $\tau$ we have $\max_\pi \Expect_{s \sim \rho} J_{1,\tau}^\pi(s) = \Expect_{s \sim \rho} J_{1,\tau}^\star(s)$, we obtain the following saddle-point problem:
\begin{equation}
\label{e-saddle}
    \max_{\pi \in \Pi} \min_{\tau \geq 0} \Expect_{s \sim \rho}  J_{1,\tau}^\pi(s)
\end{equation}
where $\Pi \subseteq \Delta(|\Ac|)^{|\Sc|}$ is some possibly restricted policy space.
The solution to this saddle-point problem yields the tightest upper-bound on the expected value of the optimal value function $\Expect_{s \sim \rho} \Expect_\phi V_1^\star(s)$ under the posterior $\phi$, and, as we shall show, it also minimizes a bound on the Bayesian regret. Implicit in the definition of the saddle-point problem is the assumption of strong duality, which we state next. This assumption holds, for instance, if $\Pi$ is convex.
\begin{restatable}{ass}{ass2}
\label{ass-strong-duality}
Strong duality holds for \eqref{e-saddle}, \ie,
\[
\min_{\tau \geq 0} \max_{\pi \in \Pi} \Expect_{s \sim \rho}  J_{1,\tau}^\pi(s) =  \max_{\pi \in \Pi} \min_{\tau \geq 0} \Expect_{s \sim \rho}  J_{1,\tau}^\pi(s).
\]
\end{restatable}
The saddle-point problem \eqref{e-saddle} is our main problem of interest, and the rest of this manuscript is dedicated to solving it and interpreting the solutions.

\subsection{The connection to Bayesian regret}
In order to provide a connection between the saddle-point problem \eqref{e-saddle} and the Bayesian regret \eqref{e-bayes-regret} let us define a few quantities of interest. First, we define a notion of \emph{optimism} for a given $\pi$ and $\tau \geq 0$
\[
\Error(\pi, \tau) \coloneqq \Expect_{s \sim \rho} \left(J_{1, \tau}^{\pi}(s) - \Expect_\phi V_1^\pi(s) \right).
\]
Since $J_{1, \tau}^{\pi}(s) \geq \Expect_\phi V_1^\pi(s)$ for all $\tau$, the $\Error$ is measuring how much `bonus' is derived from the risk-seeking exponential utility, relative to the (risk-neutral) expected value.
Let us also define a notion of `distance' from a policy to the optimal optimistic policy as the expected KL-divergence between the policies under the stationary distribution generated by $\pi$ \cite{cover2012elements}, that is,
\[
\Dist(\pi, \tau) \coloneqq  \sum_{l=1}^L \tau \Expect_{\pi} \kld{\pi_l(s, \cdot)}{\pi_{l, \tau}^\star(s, \cdot)}.
\]
It turns out that we can relate the KL-divergence and the suboptimality gap for any policy \cite{bregman_rl, mei2020global}.
\begin{restatable}{lem}{kl}[Cor.~1 \cite{bregman_rl}]
\label{l-kl}
For any $\tau > 0 $ and policy $\pi \in \Delta(\Ac)^{|\Sc|}$ we have:
\[
\Dist(\pi,  \tau) = \Expect_{s \sim \rho} (J_{1, \tau}^\star(s) - J_{1, \tau}^\pi(s)).
\]
\end{restatable}
With this we are ready to present the following decomposition.
\begin{restatable}{lem}{ldecomp}
\label{l-decomp}
Under Assumption~\ref{ass-upper-bound} we can bound the Bayesian regret in a single episode for any policy $\pi$ as
\[
\epregret(\pi, \phi) \leq \Dist(\pi, \tau) + \Error(\pi, \tau), \quad \forall \tau \geq 0.
\]
\end{restatable}
The proof is deferred to Appendix \ref{s-proofs}. This lemma shows that we can decompose the Bayesian regret bound into two terms. One term is a distance from the policy to the optimal optimistic policy, and the other term relates to the amount of optimism in the policy. Next we show how the saddle-point problem we are solving \eqref{e-saddle} relates to this decomposition.
\begin{restatable}{thm}{thmmain}
\label{t-main}
Assume~\ref{ass-upper-bound} and~\ref{ass-strong-duality}, and
let $(\pi_*, \tau_*)$ be a solution to the saddle-point problem \eqref{e-saddle}, then
\begin{align*}
\epregret(\pi_*, \phi) %
&\leq \min_{\pi \in \Pi} \Dist(\pi,  \tau_*) +  \min_{\tau} \Error(\pi_*, \tau).
\end{align*}
\end{restatable}
We defer the proof to Appendix \ref{s-proofs}. The above Theorem tells us that even though the `players' are competing in a zero-sum game, they are in a sense cooperating to minimize the Bayesian regret of the resulting policy. The solutions to the saddle-point problem \eqref{e-saddle} are each minimizing one component that contributes to the Bayesian regret bound in the decomposition we derived in Lemma \ref{l-decomp}, and ignoring the other.
In summary, we can interpret the saddle point problem as follows:
\vspace{-3mm}
\begin{itemize}
\setlength\itemsep{0em}
    \item The policy player $\pi$ is maximizing the entropy-regularized optimistic reward, where the amount of optimism is controlled by $\tau$. Equivalently, it is minimizing the expected KL-divergence to the optimal optimistic policy, and thereby minimizing one component contributing to the regret bound. 
    \item The risk-seeking player $\tau$ is balancing the reward bonus and entropy regularization in order to minimize the upper bound on the value function under $\pi$. Equivalently, it is minimizing the amount of optimism in the policy, and thereby minimizing the other component contributing to the regret bound. %
\end{itemize}
\vspace{-3mm}

Next we show a concentration result for the optimism term.
\vspace{-3mm}
\begin{restatable}{lem}{loptregret}%
\label{l-opt-regret}
Assume that the priors are layerwise-independent and that the uncertainty at each state-action decays as $\sigma^2(s,a) = \sigma^2 / n(s,a)$ for some $\sigma > 0$, where $n(s,a$) is the visitation count of the agent to $(s,a)$. Then for any sequence of policies $\pi_t$, $t=1, \ldots, N$ after $T = N L$ timesteps we have
\[
\Expect \sum_{t=1}^N\min_\tau  \Error_{\phi^t}(\pi_t, \tau) %
\leq \tilde O(\sigma \sqrt{ |\Sc||\Ac|T}),
\]
where $\tilde O$ suppresses logarithmic terms.
\end{restatable}
The proof is included in Appendix \ref{s-proofs}. Note that the above holds for any sequence of policies and has no dependence on the feasible policy set $\Pi$. This lemma tells us that under any sequence of policies, under the optimal choice of $\tau$ the expected cumulative sum of the $\Error$ terms grows sub-linearly. 
If $\Pi = \Delta(\Ac)^{|\Sc|}$, then the optimal policy satisfies $\Dist(\pi, \tau) = 0$ in the above bound and corresponds exactly to the K-learning policy in Eq.~\eqref{e-k-policy}, so have the we following corollary.%
\begin{restatable}{cor}{basic}
\label{c-basic}
Assume~\ref{ass-upper-bound} and~\ref{ass-strong-duality} and let $\Pi = \Delta(\Ac)^{|\Sc|}$. If algorithm $\mathrm{Alg}$ produces the policy that solves the saddle-point \eqref{e-saddle} for each episode $t$ then
after $T = NL$ timesteps
\[
\cumregret(\mathrm{Alg}, \phi) \leq \tilde O(\sigma \sqrt{ |\Sc||\Ac|T}).
\]
\end{restatable}
We can ensure that Assumption \ref{ass-upper-bound} holds in the case of bounded rewards, \ie, $|r| \leq 1$ a.s., by setting $\sigma = O(L)$, which recovers the bound in \citet{o2021klearning}.

\subsection{Function approximation}
In this manuscript we are interested in efficient reinforcement learning in non-tabular settings. In this case we must resort to using function approximators to parameterize the policy (or the value function) and we are interested in how well our function approximators will perform. Here we discuss the relationship between the capacity of the function approximator and the regret for our approach.

Consider the case where we are using an approximation architecture with feasible policy set $\Pi \subset \Delta(\Ac)^{|\Sc|}$ chosen such that we can guarantee that for any $\tau$
\[
\min_{\pi \in \Pi} \max_{s,l} \kld{\pi_l(s, \cdot)}{\pi^\star_{l, \tau}(s, \cdot)} \leq \epsilon / \tau,
\]
from which we have $\min_{\pi \in \Pi} \Dist(\pi, \tau) \leq \epsilon L$.
This might occur if, for instance, we have an approximation architecture that can approximate the value functions up to a small constant $\epsilon > 0$. Consider the regret of the policy $\pi_* = \argmax_{\pi \in \Pi} \min_\tau \Expect_{s \sim \rho} J_{1, \tau} ^\pi(s)$. In this case, using Theorem~\ref{t-main}, we can bound the per-episode Bayesian regret as
\[
\epregret(\pi_*, \phi) \leq \min_\tau \Error(\pi, \tau) + \epsilon L.
\]
and so the algorithm $\mathrm{Alg}_\Pi$ producing policies $\pi_*^t \in \Pi$, $t=1, \ldots, N$ enjoys bound
\begin{align*}
\cumregret(\mathrm{Alg}_\Pi, \phi)
&\leq  \tilde O(\sigma \sqrt{|\Sc||\Ac|T}) + \epsilon T,
\end{align*}
where $T = NL$ is the total number of timesteps.
In other words, we can translate the error from the function approximation directly into a regret bound when solving the saddle-point problem \eqref{e-saddle}, and richer function classes will yield better bounds.

On the other hand, consider the case where our approximation architecture is flexible enough to represent \emph{any} policy, but our algorithm for choosing the policy employs an approximation procedure, such as online policy gradient. In that case the KL-divergence from the current policy to the optimistic policy is not zero, but if the policy is converging towards the optimal policy at some rate, then we may be able bound the sum of the KL divergences. There has been much recent work examining the convergence rate of policy gradient and entropy regularized policy gradient (under somewhat restrictive assumptions on the initial state distribution $\rho$) \cite{agarwal2021theory, zhang2021sample, bhandari2019global, bhandari2021linear, mei2020global}. We leave to future work combining the results in this paper with results from the literature for the derivation of regret bounds in that case.

\section{Epistemic-Risk-Seeking Actor-Critic}
We have a derived a two-player zero-sum game, the solution of which yields a policy that explores efficiently by minimizing a bound on Bayesian regret. There are many possible approaches one could use to solve the saddle point problem, even in the purely online RL setting. In this section we describe a very simple approach that works reasonably well in practice, though it is likely that more sophisticated variants of policy algorithms would perform better \cite{schulman2017proximal, map2018, schulman2015trust,kakade2001natural}. Our approach is to derive gradients for both the policy parameters and the risk-seeking parameter, then to update them online simultaneously using stochastic gradients. %
If we parameterize the policy $\pi$ by some $\theta \in \Theta$, then the gradient of the saddle-point problem~\eqref{e-saddle} with respect to $\theta$ is given by
\begin{align}
\label{e-pol-grad}
\Expect_{s\sim\rho} \nabla_\theta J_{1,\tau}^\pi(s) =& 
\sum_{l=1}^L \Expect_{\pi} \Big(\nabla_\theta \log \pi(s_l, a_l) K_{l, \tau}^\pi(s_l,a_l) +\\&\quad  \tau \nabla_\theta H(\pi(s_l, \cdot))\Big).
\end{align}
This is a straightforward extension of the classic policy gradient theorem adapted to our case \cite{sutton1999policy}.
Similarly, the gradient with respect to $\tau$ is given by
\begin{align*}
    \Expect_{s\sim\rho} \nabla_\tau J_{1,\tau}^\pi(s)
    = \sum_{l=1}^L \Expect_{\pi} \left( H(\pi(s_l, \cdot) - \frac{\sigma^2(s_l,a_l)}{2 \tau^2}\right).
\end{align*}
Finally, we have a relationship between the gradients and the Bayesian regret bound decomposition.
\begin{restatable}{cor}{gradsdecomp}
The gradients of the saddle-point correspond to the gradients of the components in the Bayesian regret decomposition in Lemma \ref{l-decomp}, \ie,
\begin{align*}
    &\Expect_{s\sim\rho} (-\nabla_\theta J_{1,\tau}^\pi(s), \nabla_\tau J_{1,\tau}^\pi(s)) =\\ &\quad \quad (\nabla_\theta \Dist(\pi, \tau), \nabla_\tau \Error(\pi, \tau)).
\end{align*}
\end{restatable}
Fixing $\tau$ and taking a step in the negative gradient with respect to $\theta$ is towards minimizing the KL distance to the optimal optimistic policy, and for fixed $\theta$ taking a step in the direction of the gradient with respect to $\tau$ is towards minimizing the amount of optimism in the policy. Seen this way, the gradient flow is in the direction of minimizing the components in the Bayesian regret bound decomposition from Lemma \ref{l-decomp}. %

Importantly, both of these gradient terms can be interpreted as expectations under the state-action distribution induced by the policy $\pi$.
This suggests a scheme where we sample states and actions from the distribution generated by the policy, and use the same samples to update both quantities. We call this approach \emph{epistemic-risk-seeking actor-critic} (ERSAC), and it is implemented as Algorithm~\ref{a-kac} (presented in the appendix). Since this algorithm is applying stochastic gradient ascent-descent, rather than solving the saddle-point problem \eqref{e-saddle} exactly, we have no known Bayesian regret guarantees. However, as we shall demonstrate empirically, this algorithm tends to perform significantly better than vanilla actor-critic in hard exploration problems.

\paragraph{Estimating the uncertainty $\sigma$.}
In Algorithm \ref{a-kac} we left the process of deriving the estimator of $K^\pi$ open. An estimator that performed well in practice is to use online TD$-\lambda$ with $\lambda = 0.8$ and a rollout length of $N=50$ \cite{sutton:book}.%
We have also left the source of the uncertainty signal $\sigma(s,a)$ undefined. There is much work in the deep RL literature that could be plugged into the algorithm here as discussed in \S\ref{s-intro}. For our experiments we augmented the neural network with an ensemble of reward prediction heads with randomized prior functions \cite{osband2018randomized}, and used the variance of the ensemble predictions as the uncertainty signal.

\paragraph{Comparison to other actor-critic methods.}
Algorithm \ref{a-kac} is a relatively small modification of a vanilla actor critic, the modifications are in blue. They are primarily the addition of the uncertainty terms, learning the $\tau$ risk-seeking parameter, and the addition of entropy regularization weighted with the learned $\tau$. In our experiments we shall refer to the algorithm without these modifications as \emph{vanilla actor-critic}. The presence of the reward predictors in Algorithm \ref{a-kac} can act as an auxiliary task and potentially improve the representation learned by the neural network thereby improving performance. This would give our agent an advantage over vanilla actor-critic that has nothing to do with exploration. To counter that, we also give the vanilla actor-critic agent the same reward prediction task, but we do not use the uncertainty estimates they generate.

A common pattern in optimistic deep RL algorithms is to simply add an optimism bonus to the rewards based on the standard deviation of the uncertainty, \ie, replace the reward with $r^+(s,a) = \bar r(s,a) + \mu \sigma(s,a)$ for some hyper-parameter $\mu > 0$, and then run a vanilla actor-critic algorithm using this reward. In our experiments we shall refer to this variant as \emph{simple optimism actor-critic}, where the uncertainty signal is the same ensemble approach as used by Algorithm \ref{a-kac} and all results are presented after tuning the $\mu$ hyper-parameter.

\section{DeepSea Numerical Results}
\begin{figure*}[!ht]
\begin{center}
\includegraphics[scale=0.23]{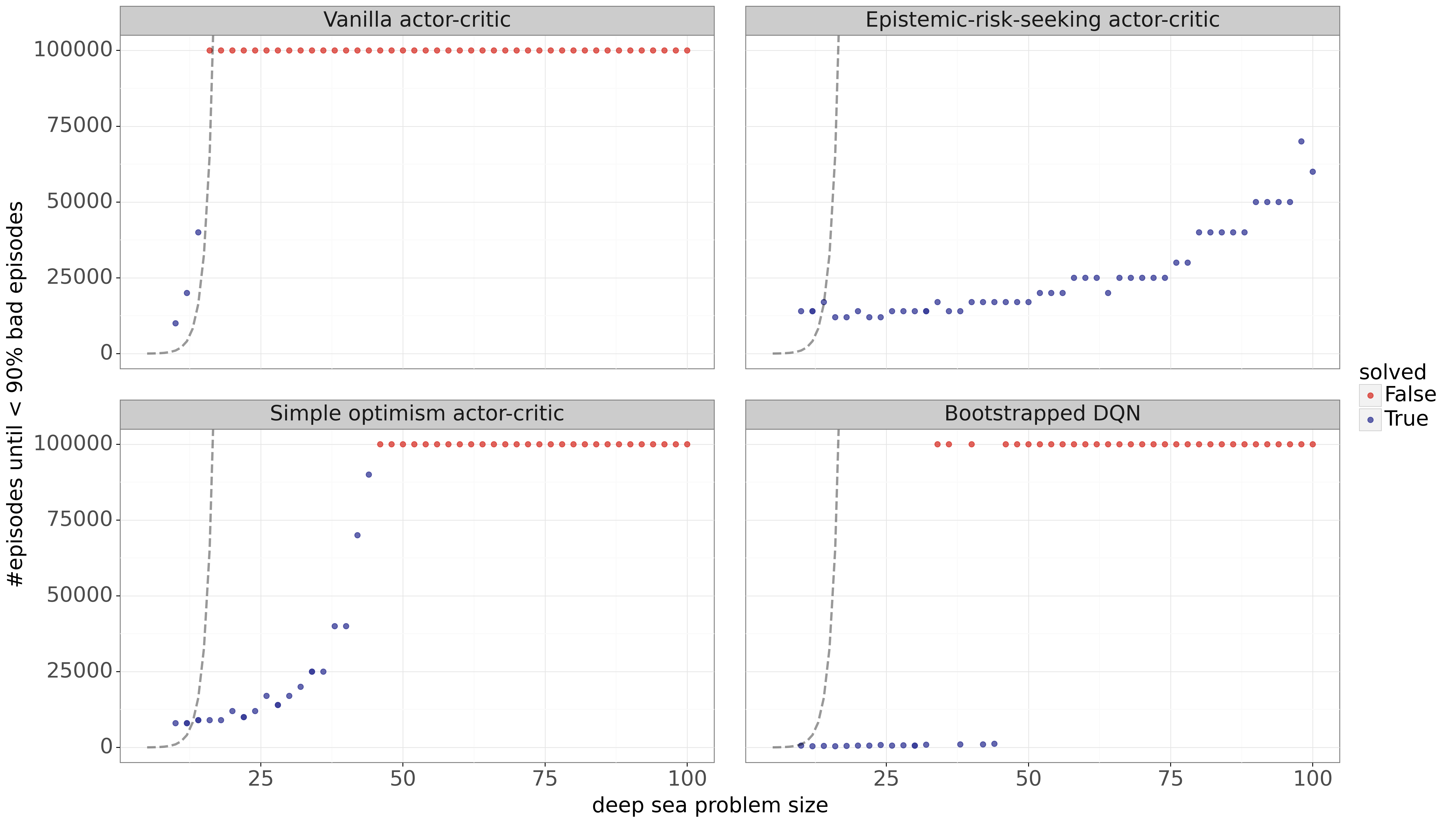}
  \caption{ERSAC is able to solve far deeper DeepSea instances than Bootstrapped DQN, despite requiring significantly less compute. Adding simple optimism to actor-critic provides some benefit, but it struggles to solver deeper instances. Vanilla actor-critic requires exponential experience to solve DeepSeas of increasing depth.\vspace{-6mm}}
  \label{f-DeepSea_solved_v_boot_dqn}
\end{center}
\end{figure*}

\begin{figure}[!h]
\begin{center}
\includegraphics[scale=0.25]{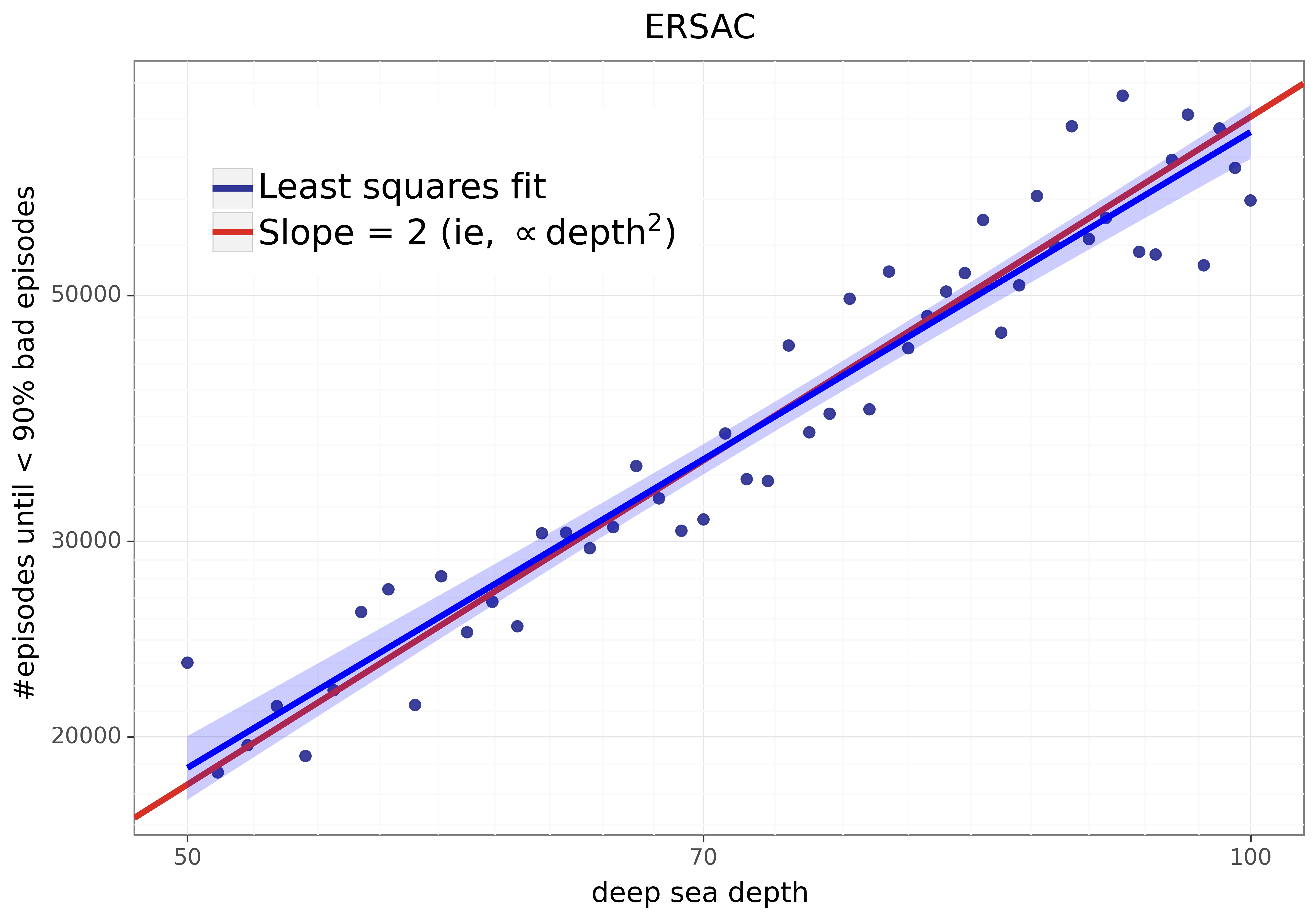}
  \caption{Algorithm~\ref{a-kac} has an empirical quadratic dependency on depth when solving DeepSea.\vspace{-4mm}}
  \label{f-solve_time_v_depth}
\end{center}
\end{figure}

In the DeepSea environment the agent finds itself at the top left of an $L \times L$ grid and must navigate it to find the reward in the bottom right corner, see Figure \ref{f-DeepSea}. At each time-step the agent descends one row and must choose to move one column left or right. This is a challenging exploration unit-test because the agent needs to select the action `move right' $L$ times in a row in order to reach the goal \cite{osband2019behaviour} (in practice the actions corresponding to right and left are different in each state to prevent an agent with a bias for taking one action repeatedly from solving the problem unfairly). An agent that is acting randomly will take time exponential in $L$ to reach the goal. However, agents that are exploring efficiently should reach the goal in time \emph{polynomial} in $L$. Although DeepSea can be made a tabular environment, in this experiment we feed a one-hot representation of the agent location into a neural network in order to test how various deep RL approaches work. We compare four approaches: Vanilla actor-critic, ERSAC (Alg.~\ref{a-kac}) with a reward predictor ensemble size of $10$, simple optimism actor-critic with the same uncertainty signal as Alg~\ref{a-kac}, and Bootstrapped DQN \cite{osband2016deep} with $10$ elements in the value ensemble and $10$ randomized priors (one per ensemble member). All agents had the same basic network architecture. Bootstrapped DQN performs an update with batch size of $128$ samples every actor step which is the default in the agent implemented in the `bsuite' \cite{osband2019behaviour}. This uses substantially more compute and wall-clock time than the other approaches, and required a GPU to run efficiently. In Figure \ref{f-DeepSea_solved_v_boot_dqn} we show the results of the four approaches. In that figure the blue dots represent solved DeepSea instances (where solved means the agent reached the goal reliably) and red dots are unsolved. The $x$-axis is depth and the $y$-axis is the number of episodes until that depth is solved. The grey dashed line is exponential in depth, which is the dependence we expect a naive agent to have. If the agent is consistently below this line, then it is exploring well. 

As we can see, the naive vanilla actor-critic algorithm suffers from an exponential dependence on depth and consequently cannot solve depths of greater than around $14$ within $10^5$ episodes. Bootstrapped DQN is much faster than Algorithm \ref{a-kac} at learning the small DeepSea instances since it uses significant amounts of replay (though we close this gap in \S \ref{s-off-policy}). However, as the DeepSea size grows it suddenly fails, unable to solve DeepSea instances larger than around size 50. The simple optimism actor-critic does provide some benefit over vanilla actor-critic, as it is able to solve DeepSeas out to approximately depth 50, however, the dependency on depth is significantly worse than ERSAC. ERSAC (Algorithm \ref{a-kac}) is able to solve DeepSea instances out to size 100 without a clear performance degradation. In Figure \ref{f-solve_time_v_depth} we show on a log-log plot that Algorithm \ref{a-kac} has an empirical \emph{quadratic} dependency on depth, a major improvement over the exponential dependency of the naive actor-critic approach.

In Appendix \ref{s-app-DeepSea} we further analyze the performance on DeepSea, and the sensitivity of Algorithm \ref{a-kac} to various hyper-parameters. We also test far deeper DeepSeas, including showing performance on a DeepSea of depth 250 where $99$ out of $100$ seeds reached the goal with $10^6$ episodes. To the best of our knowledge no other deep RL algorithm has been able to solve such hard instances of DeepSea.

\section{Incorporating Off-Policy Data}
\label{s-off-policy}
So far our discussion of Algorithm~\ref{a-kac} has been entirely about the on-policy case. In practice however, state-of-the-art deep RL agents use a substantial amount of experience replay data, which vastly improves data efficiency and overall performance \cite{mnih-dqn-2015, o2016pgq, hessel2018rainbow}. Since exploration is also about increasing data efficiency, being able to combine replay and principled exploration would yield a double improvement. In this section we extend Algorithm~\ref{a-kac} to use off-policy replay data and show that combining the risk-seeking objective and replay can provide large performance improvements. To do that we make the following updates to the core algorithm:
\vspace{-2ex}
\begin{itemize}
\itemsep0em 
    \item Add state-action-reward-noise $(s_t, a_t, r_t, \zeta_t)$, $t=1,2,\ldots,$ transition data to a replay buffer, where $\zeta_t \sim \mathcal{N}(0, \rho I_K)$ is independent noise with variance $\rho \geq 0$, and $K$ is the size of the ensemble.
    \item Mix on-policy data with off-policy data sampled from the replay buffer according to a prioritization scheme \cite{schaul2015prioritized}.
    \item Apply V-trace clipped importance sampling corrections to the off-policy trajectories \cite{espeholt2018impala}.
    \item Use the reward + noise as targets for the reward prediction ensemble \cite{dwaracherla2022ensembles}.
\end{itemize}
\vspace{-2ex}
The above setup adds a small amount of Gaussian noise to the targets for the reward ensemble. This is necessary to prevent collapsing the uncertainty estimates from the use of replay data. It is important that the noise terms be added to the replay since this ensures that the epistemic uncertainty decays with the number of real data, rather than the number of replay steps. Using randomly initialized reward heads, randomized prior functions, and adding noise to the replay buffer (a form of Bayesian bootstrapping) follows the recipe analyzed in \citet{dwaracherla2022ensembles} for good uncertainty estimates using ensembles. The V-trace clipped importance sampling re-weights the data coming from off-policy data according to how likely it is under the \emph{current} policy, so that the gradient update in \eqref{e-pol-grad} is still (approximately) under the correct measure when using replay \cite{munos2016safe, espeholt2018impala}. We shall refer to Algorithm~\ref{a-kac} when we add the changes above as `ERSAC + replay'.

\begin{figure}[!h]
\begin{center}
\includegraphics[scale=0.24]{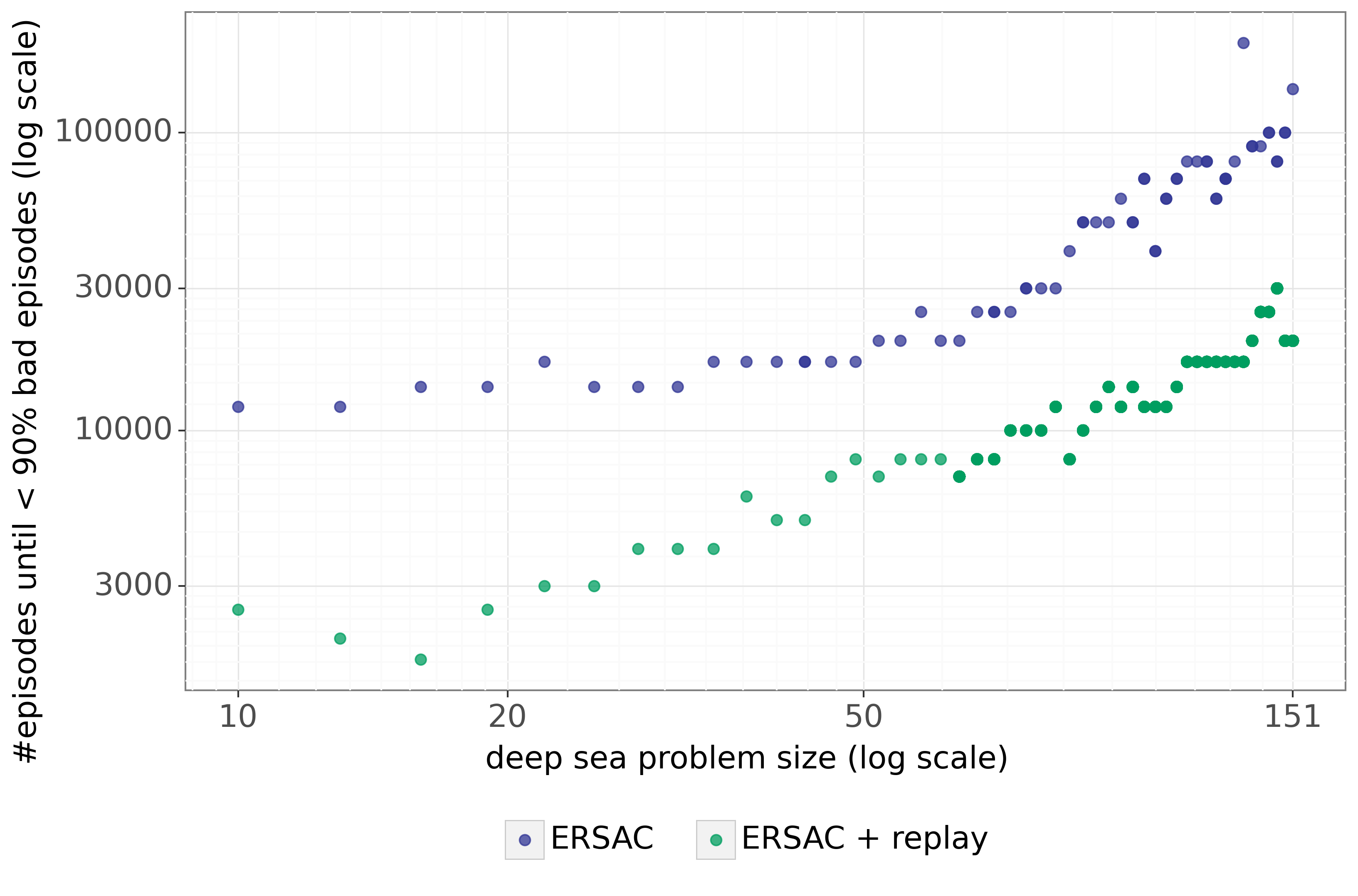}
  \caption{Adding replay to ERSAC improves data efficiency by a factor of about $4\times$ on DeepSea. Note the depth here goes to $151$.\vspace{-4mm}}
  \label{f-replay_v_no_replay}
\end{center}
\end{figure}

In Fig.~\ref{f-replay_v_no_replay} we compare the performance of ERSAC on DeepSea both with and without replay data. It is clear that adding replay data substantially improves performance, while maintaining the empirical quadratic dependence of solve time on depth (see Fig.~\ref{f-solve_time_v_depth_kimpala}). Overall, the ERSAC + replay agent yields about a $4\times$ data efficiency improvement over the pure on-policy version. The off-policy agent here used a batch size of $16$ with an offline-data fraction of $0.97$ per batch. Replay was prioritized by TD-error and when sampling the replay prioritization exponent was $1.0$ \cite{schaul2015prioritized}. The replay noise parameter was $\rho=0.1$. All other settings were identical to the on-policy variant.  In order to show the advantage of using noise in the replay buffer, we show results with and without noise on DeepSea in Figure \ref{f-replay_noise_v_no_noise}.

There are two main ways in which replay may improve performance on DeepSea. First, it may reach the goal faster. Second, once the goal is reached it may `latch on' faster, that is it may return to the goal consistently in fewer episodes. In Fig.~\ref{f-depth_100_kimpala_100_seeds} we show the reward of the agents on a depth $100$ instance of DeepSea, averaged over $100$ random seeds. It is clear that using replay is \emph{both} finding the goal and latching on faster. However, we note that the on-policy version of the algorithm reached the goal 99 times out of 100, whereas the off-policy version reached the goal only 93 times. This suggest that the replay version may not be quite as robust as the on-policy version, at least for the hyper-parameters we used.

\begin{figure}[!h]
\begin{center}
\includegraphics[scale=0.35]{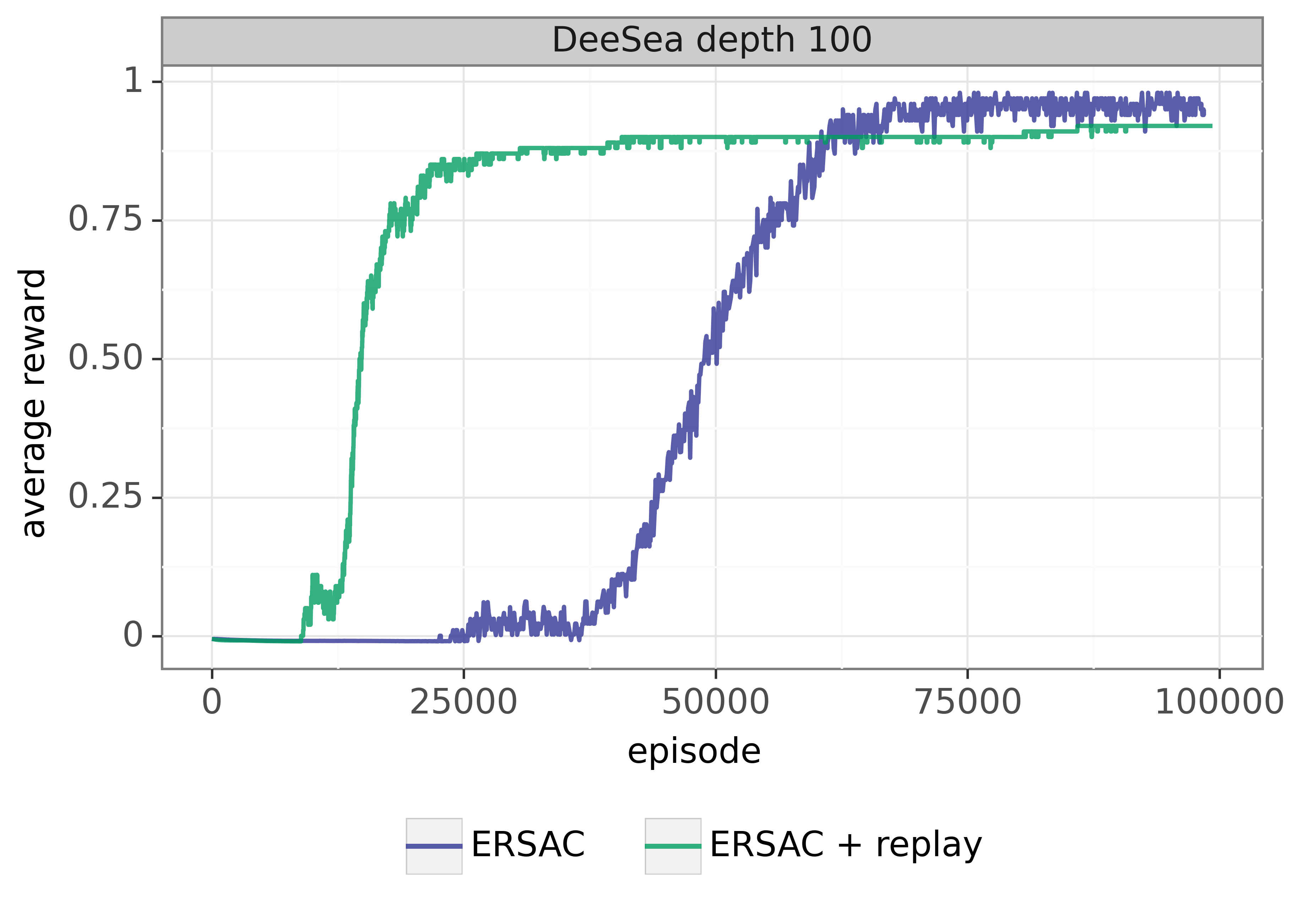}
  \caption{Adding replay to ERSAC improves both the time until goal first reached and the latching on speed in DeepSea.\vspace{-5mm}}
  \label{f-depth_100_kimpala_100_seeds}
\end{center}
\end{figure}

\section{Atari Numerical Results}
Finally, we compare ERSAC + replay to an Actor-critic + replay agent on the Atari benchmark \cite{bellemare-ale}. Our setup involves actors generating experience and sending them to a learner, which mixes the online data and offline data from a replay buffer to update the network weights \cite{hessel2021podracer, mnih2016asynchronous}. Our agent is relatively simple compared to modern state-of-the-art Atari agents since it is missing components like model-based rollouts, distributional heads, auxiliary tasks, \etc~ The point of these experiments is not to produce state-of-the-art results, but to provide evidence of a clear benefit when the addition of the risk-seeking objective function is incorporated into a policy-gradient based agent. We ran both agents on the full Atari suite and averaged the results over five seeds. Between the agents all hyper-parameters in common were set to the same values, and tuned for the replay actor-critic agent performance. The replay actor-critic agent used a fixed entropy regularization of $0.02$. The per-game results for all 57 games are presented in Figure \ref{f-atari_all_games}, and Figure~\ref{f-atari_human_normalized_median} shows the median human-normalized performance across the entire suite (calculated in the same way as \citet{hessel2018rainbow}). Clearly the addition of the risk-seeking objective in Algorithm \ref{a-kac} is providing a significant benefit over the actor-critic agent. The ERSAC agent reaches the peak performance of the actor-critic agent in about $1.8\times$ fewer environment frames, for essentially the same computational cost. The advantage comes from the fact that the risk-seeking objective leads to deep exploration, which results in finding higher rewarding states and in better cumulative performance.

\begin{figure}[!h]
\begin{center}
\includegraphics[scale=0.34]{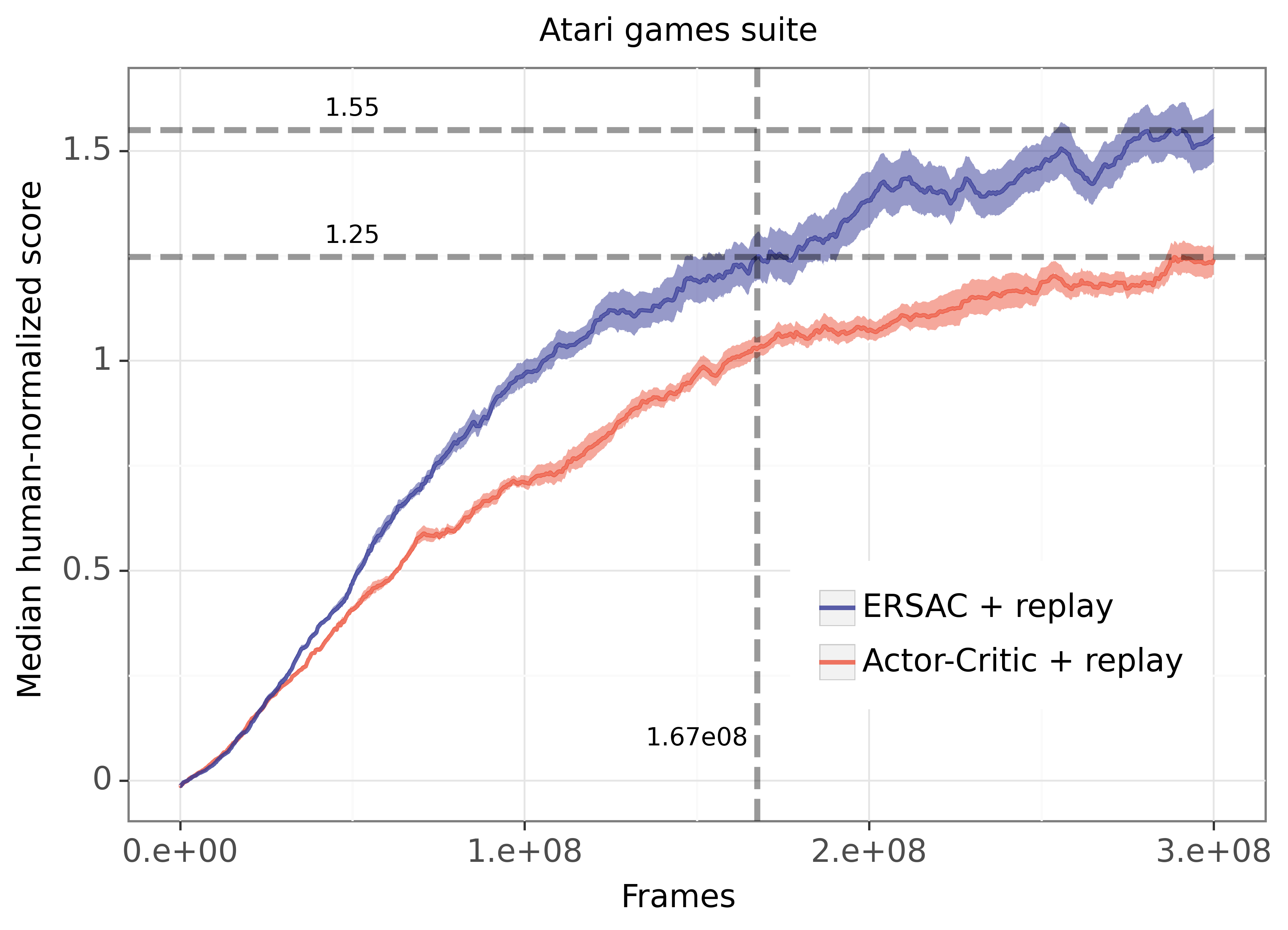}
  \caption{ERSAC reaches the same median performance on the Atari suite as the actor-critic baseline in about $1.8\times$ fewer frames.\vspace{-5mm}}
  \label{f-atari_human_normalized_median}
\end{center}
\end{figure}

\section{Conclusion}
We presented a new policy-gradient algorithm for efficient exploration. It was derived by endowing the agent with an epistemic-risk-seeking utility function, where the amount of risk-seeking is controlled by a risk-seeking parameter. The formulation entails solving a zero-sum game between the policy and the risk-seeking parameter. The policy is updated to maximize the optimistic reward, and the risk-seeking parameter is tuned to minimize regret. This procedure is a small modification to vanilla actor-critic but produces vastly improved results on challenging exploration problems. %

\newpage

\bibliographystyle{icml2023}
\bibliography{refs}

\newpage
\appendix
\onecolumn

\section{Main algorithm}

\begin{algorithm}
\caption{Epistemic-risk-seeking actor-critic (ERSAC)}
\begin{algorithmic}[1]
\label{a-kac}
  \STATE Input initial parameters $\theta^0 \in \Theta$, {\color{blue} $\tau^0 > 0$, uncertainty estimator
  $\sigma: \Sc \times \Ac \rightarrow \reals_+$}
  \STATE Input policy function $\pi_\theta: \Sc \rightarrow \Delta(\Ac)$ and value function $J_\theta: \Sc \rightarrow \reals$
  \STATE For $k=0, 1, \ldots$
  \STATE \quad Gather trajectory $\nu = (r_1, s_1, a_1,\ldots, r_N, s_N)$ using $\pi^k$
  \STATE \quad {\color{blue} Compute uncertainties $\sigma(s_i, a_i)$, $i=1, \ldots N$}
  \STATE \quad Estimate {\color{blue} $\hat K_{l, \tau}^\pi(s_i, a_i)$} using $r_i, J_{\theta^k}(s_i), {\color{blue} \sigma(s_i, a_i), \tau^k}$, $i=1, \ldots, N$, and Eq.~\ref{e-kpi}
  \STATE \quad $L_\mathrm{policy} =  (1/N) \sum_{i=1}^N \left(\log \pi_{\theta^k}(s_i, a_i) \verb|stop_grad|(\hat K_{l, \tau}^\pi(s_i, a_i) - J_{\theta^k}(s_i)) - \tau^k H(\pi_{\theta^k}(s_i, \cdot)) \right)$
  \STATE \quad $L_\mathrm{value} = (1/N) \sum_{i=1}^N (J_{\theta^k}(s_i) - \verb|stop_grad|(\hat K_{l, \tau}^\pi(s_i, a_i) - \tau^k \log \pi_{\theta^k}(s_i, a_i)))^2$
  \STATE \quad {\color{blue} $L_\tau = (1/N) \sum_{i=1}^N \left(\frac{\sigma^2(s_i, a_i)
  }{2 \tau} + \tau H(\pi_{\theta^k}(s_i, \cdot) \right)$}
  \STATE \quad $\theta^{k+1} = \theta^k + \eta (\nabla_\theta L_\mathrm{policy} - \nabla_\theta L_\mathrm{value}) $
  \STATE \quad {\color{blue} $\tau^{k+1} = \tau^k - \eta \nabla_\tau L_\mathrm{\tau}$} %
  \STATE \quad {\color{blue} Update uncertainty model $\sigma$ using $\nu$}
\end{algorithmic}
\end{algorithm}

\section{Proofs}
\label{s-proofs}
\ldecomp*
\begin{proof}
\begin{align*}
\epregret(\pi, \phi) &= \Expect_\phi \Expect_{s \sim \rho} (V_1^\star(s) - V_1^{\pi}(s))
\\
&\leq \Expect_{s \sim \rho} (J_{1, \tau}^{\star}(s) - \Expect_\phi V_1^{\pi}(s))
\\
&= \Expect_{s \sim \rho} (J_{1, \tau}^{\star}(s) - J_{1, \tau}^{\pi}(s) + J_{1, \tau}^{\pi}(s)- \Expect_\phi V_1^{\pi}(s))
\\
&= \Dist(\pi, \tau) +  \Error(\pi, \tau).
\end{align*}
\end{proof}

\thmmain*
\begin{proof}
We can rewrite the saddle-point formulation in two ways. For any $\pi$ we have
\begin{align*}
 \min_\tau \Expect_{s \sim \rho} J_{1, \tau} ^\pi(s) %
 &= \min_\tau \Expect_{s \sim \rho} (J_{1, \tau}^{\pi}(s) - \Expect_\phi V_1^\pi(s) + \Expect_\phi V_1^\pi(s)) \\
 &= \Expect_{s \sim \rho} \Expect_\phi V_1^\pi(s) + \min_\tau \Error(\pi, \tau),
\end{align*}
and for any $\tau$
\begin{align*}
 \max_{\pi \in \Pi} \Expect_{s \sim \rho} J_{1, \tau} ^\pi(s)%
 &= \max_{\pi \in \Pi} \Expect_{s \sim \rho} (J_{1, \tau}^{\pi}(s) - J_{1, \tau}^\star(s) + J_{1, \tau}^\star(s)) \\
 &= \Expect_{s \sim \rho} J_{1, \tau}^\star(s) - \min_{\pi \in \Pi} \Dist(\pi, \tau).
\end{align*}
From strong duality and the fact that $(\pi_*, \tau_*)$ is a primal-dual optimum we know that $\max_{\pi \in \Pi} \Expect_{s \sim \rho} J_{1, \tau_*} ^{\pi}(s) = \Expect_{s \sim \rho} J_{1, \tau_*} ^{\pi_*}(s)= \min_{\tau}\Expect_{s \sim \rho} J_{1, \tau} ^{\pi_*}(s)$, which implies
\begin{align*}
\Expect_{s \sim \rho} J_{1, \tau_*}^\star(s) - \min_{\pi \in \Pi} \Dist_{\phi}(\pi, \tau_*) = \Expect_{s \sim \rho} \Expect_\phi V_1^{\pi_*}(s) + \min_\tau \Error(\pi_*, \tau)
\end{align*}
and so
\begin{align*}
    \epregret(\pi_*, \phi) &\leq \Expect_{s \sim \rho} (J_{1, \tau_*}^\star(s) -
    \Expect_\phi V_1^{\pi_*}(s) )\\
    &= \min_{\pi \in \Pi} \Dist(\pi, \tau_*) + \min_\tau \Error(\pi_*, \tau).
\end{align*}
\end{proof}

\loptregret*
\begin{proof}
At episode $t$ denote the uncertainty at state-action $(s,a)$ as $\sigma^2 / n^t(s,a)$, where $n^t(s,a)$ is the visitation count of $(s,a)$ before episode $t$. Under the assumption of independent priors across layers we can write
\[
\Expect_{s \sim \rho} \Expect_{\phi^t} V_1^{\pi^t}(s) = \sum_{l=1}^L \Expect_{\pi^t} \bar r_l^t(s_l,a_l),
\]
and recall that 
\[
  \Expect_{s \sim \rho}  J_{1,\tau}^{t,\pi^t}(s) =
  \sum_{l=1}^L \Expect_{\pi^t} \left(\bar r^t_l(s_l, a_l) +  \frac{\sigma_l^2}{2 \tau n^t(s_l, a_l)} +  \tau H(\pi^t_l(s_l, \cdot))\right),
\]
and so we can write
\begin{align*}
\Error(\pi_t, \tau) &= \Expect_{\phi} \Expect_{s \sim \rho} (J_{1, \tau}^{\pi}(s) - V_1^{\pi}(s)) \\
&= \sum_{l=1}^L \Expect_\pi \left( \frac{\sigma^2}{2 \tau n^t(s_l, a_l)} +  \tau H(\pi^t_l(s, \cdot))\right).
\end{align*}
Now define scalar (up to log factors which we shall ignore for brevity)
\[
\tau_N = \tilde O(\sigma \sqrt{|\Sc| |\Ac| /  (L N)  }).
\]
Then we have
\begin{align*}
\sum_{t=1}^N \min_\tau \Error_{\phi^t}(\pi_t, \tau)
    &\leq  \sum_{t=1}^N \Error_{\phi^t}(\pi_t, \tau_N)
    \\
    &= \sum_{t=1}^N \sum_{l=1}^L \Expect_\pi \left( \frac{\sigma^2}{2 \tau_N n^t(s_l, a_l)} + \tau_N H(\pi^t_l(s, \cdot))\right)
    \\
    &\leq (1/2) \sigma^2 |\Ac|( 1 + \log N) \tau_N^{-1} \sum_{l=1}^L |\Sc_l| + \tau_N N L \log |\Ac|
    \\
    &\leq \tilde O(\sigma \sqrt{L |\Sc||\Ac| N}),
\end{align*}
where we used the fact that entropy is bounded, the pigeonhole principle Lemma 6 from \cite{o2021klearning}, and the identity $\sum_{l=1}^L |\Sc_l| = |\Sc|$. The result follows by substituting in $T = N L$.
\end{proof}

\section{DeepSea results discussion}
\label{s-app-DeepSea}

\begin{figure}[h]
\begin{center}
\includegraphics[scale=0.4]{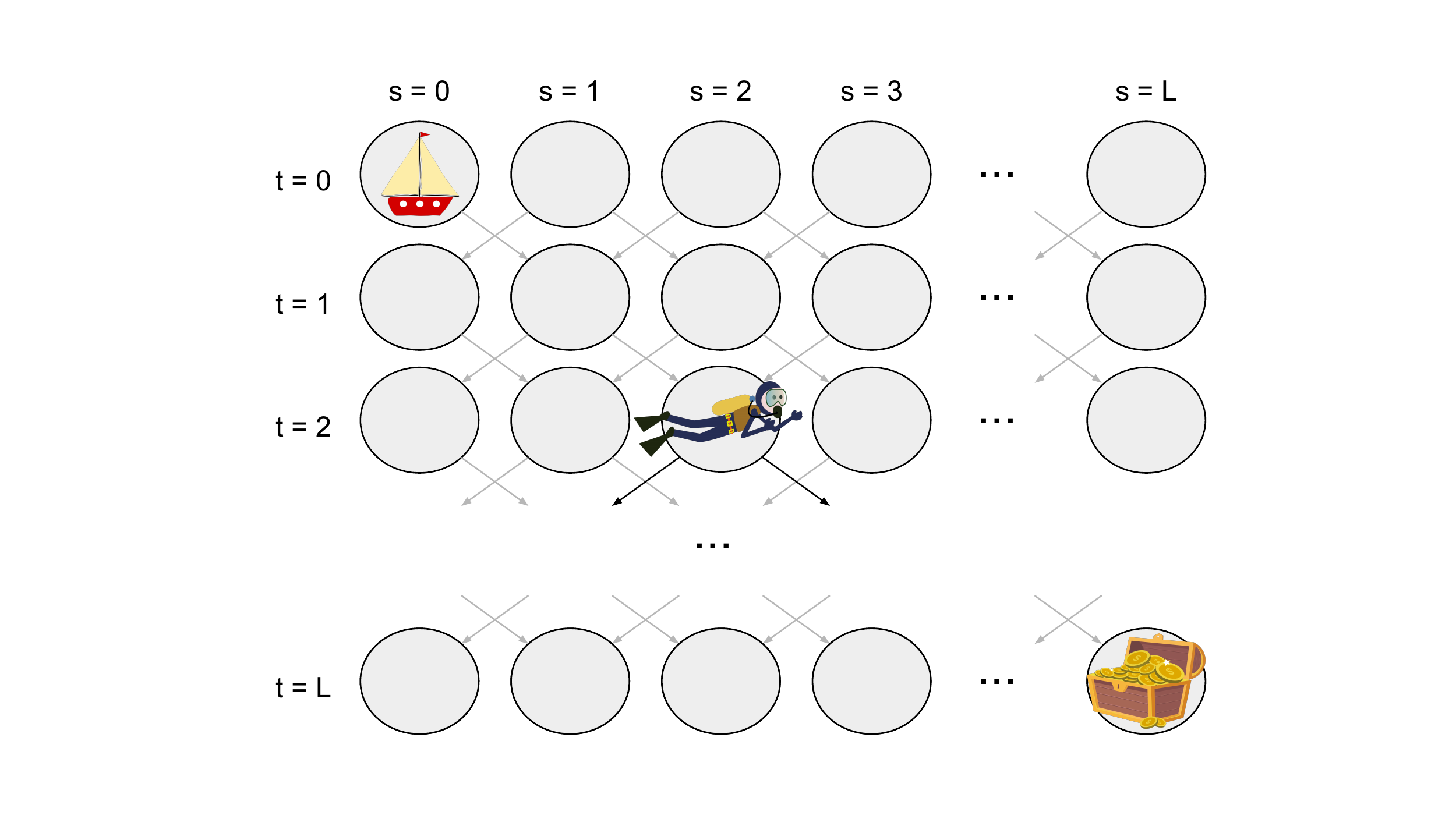}
  \caption{The DeepSea MDP is a challenging exploration `unit-test' where the agent must navigate from the top left state to the bottom right in order to collect a positive reward. Naive exploration approaches take time exponential in depth to solve this problem.
  }
  \label{f-DeepSea}
\end{center}
\end{figure}

K-learning has a worst-case $\tilde O(L \sqrt{|\Sc| |\Ac| T})$ Bayesian regret in tabular domains. In a DeepSea of depth $d$ we have $L=d$, $S=d^2$, $A=2$, so this regret bound would translate as $\tilde O(d^2 \sqrt{T})$, and to have \emph{average} regret below some threshold would require $O(d^4)$ timesteps, or $O(d^3)$ episodes. Algorithm \ref{a-kac} is an online, stochastic policy gradient based approximation to K-learning, so we have no known regret bound guarantee. However, in Figure (\ref{f-solve_time_v_depth}) we find that empirically for Algorithm \ref{a-kac} the number of episodes required to `solve' a DeepSea instance appears to have a quadratic dependency on depth, a factor of $d$ better than the worst-case bound. Naive approaches to exploration require episodes scaling as $O(2^d)$, so a quadratic dependency is a substantial improvement.

Our agent used TD-$\lambda$ with $\lambda=0.8$, Figure \ref{f-robustness_lambda} we show the performance of the agent as a function of the $\lambda$ parameter. It appears that values of $\lambda \geq 0.6$ perform well, able to solve most or all of the DeepSea instances out to depth 100. Figure \ref{f-robustness_rollout} shows the robustness of the method to TD-$\lambda$ rollout length. For very small rollouts (\eg, 1) the benefit of Algorithm \ref{a-kac} over vanilla actor-critic is minor, however the epistemic-risk-seeking agent is able to solve practically all DeepSea instances to depth 100 reliably for just a rollout of length 25.

\begin{figure}[h!]
\begin{center}
    \includegraphics[scale=0.3]{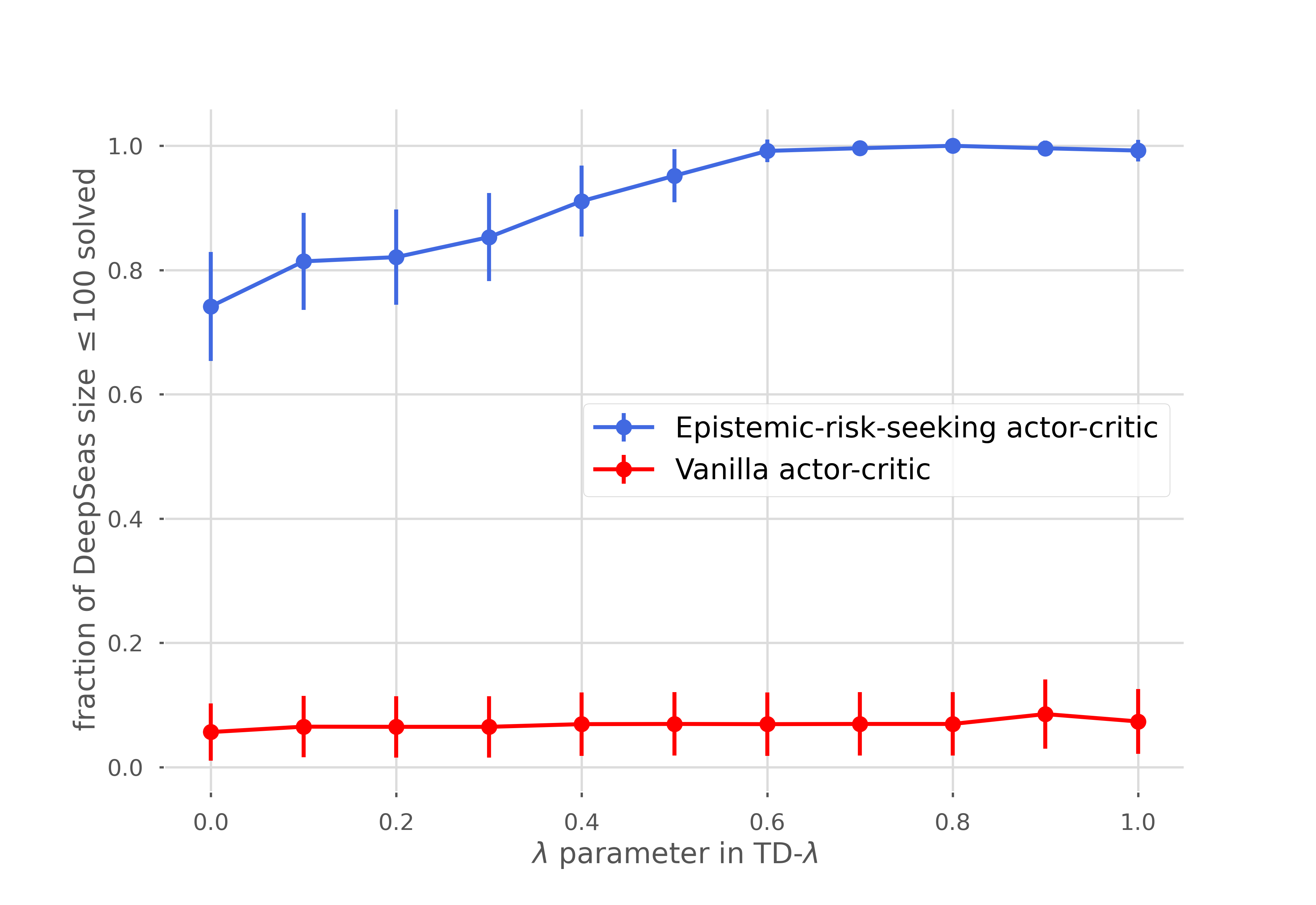}
  \caption{When using TD-$\lambda$ to estimate the K-values in Algorithm \ref{a-kac} larger $\lambda$ values tend to perform better in DeepSea.}
  \label{f-robustness_lambda}
\end{center}
\end{figure}

\begin{figure}[h!]
\begin{center}
    \includegraphics[scale=0.3]{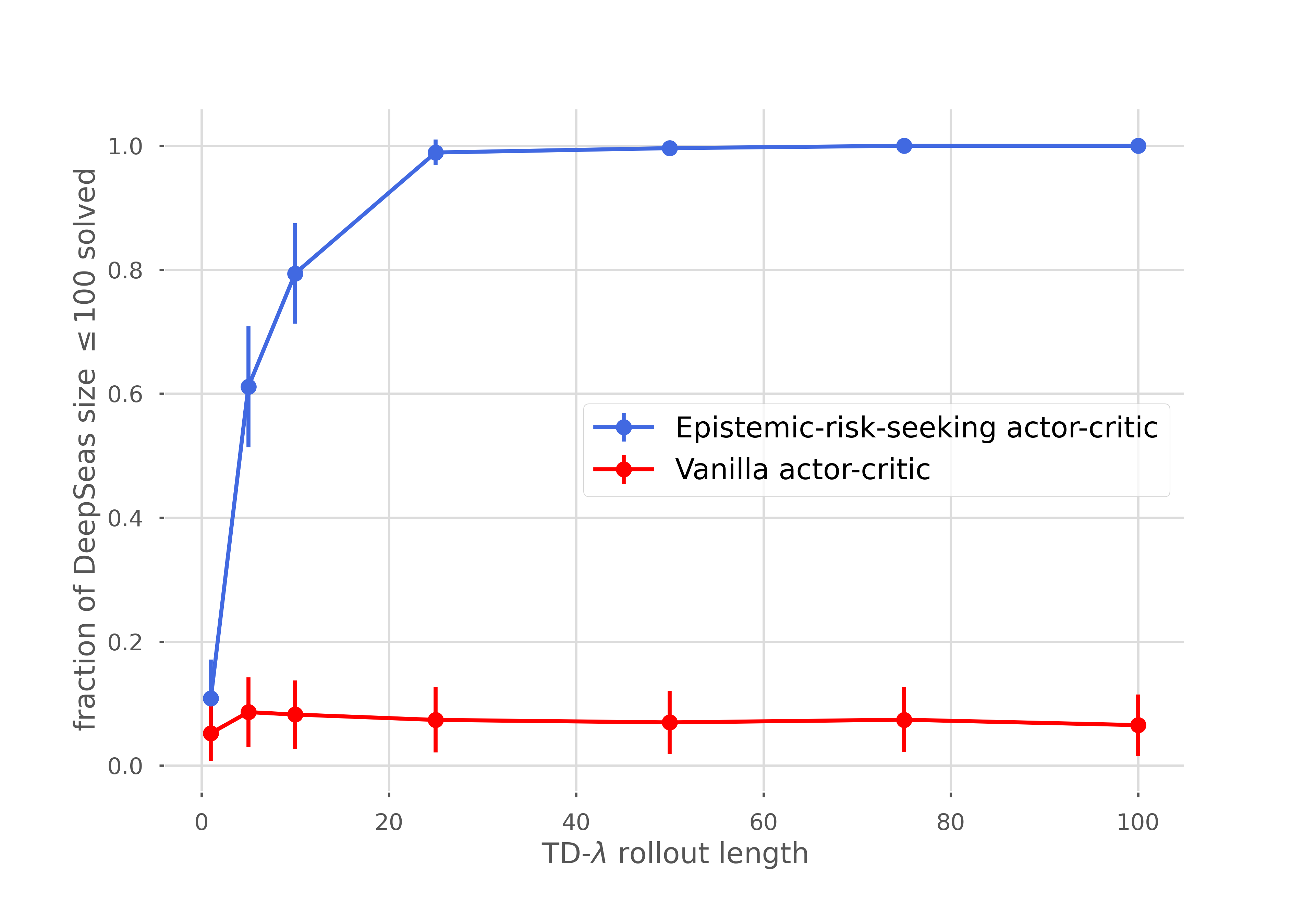}
  \caption{When using TD-$\lambda$ to estimate the K-values in Algorithm \ref{a-kac} even relatively small rollout lengths are able to solve deeper DeepSea instances. Even a rollout length of $5$ is able to solve more than $60\%$ of DeepSea instances, and $25$ is enough to solve practically all of them.}
  \label{f-robustness_rollout}
\end{center}
\end{figure}

Finally, we also tested how important \emph{learning} the risk-seeking parameter $\tau$ is, as done in Algorithm \ref{a-kac}. In Figure \ref{f-robustness_tau} we compare the approach in Algorithm \ref{a-kac} to simply using a fixed $\tau$ parameter. From this Figure it appears that there is a fixed choice of $\tau$ that matches the learned approach on DeepSea performance. However, the performance of the agent is highly dependent on this parameter and even small deviations can dramatically degrade reliability. On the other hand, Algorithm \ref{a-kac}, which learns $\tau$ from data, is able to solve almost all the DeepSea instances robustly over a wide range of initial choices of $\tau$, which suggests that the update rule that minimizes the zero-sum game \eqref{e-saddle} over $\tau$ is effective at tuning the amount of risk-seeking for efficient exploration.
\begin{figure}[h]
\begin{center}
    \includegraphics[scale=0.3]{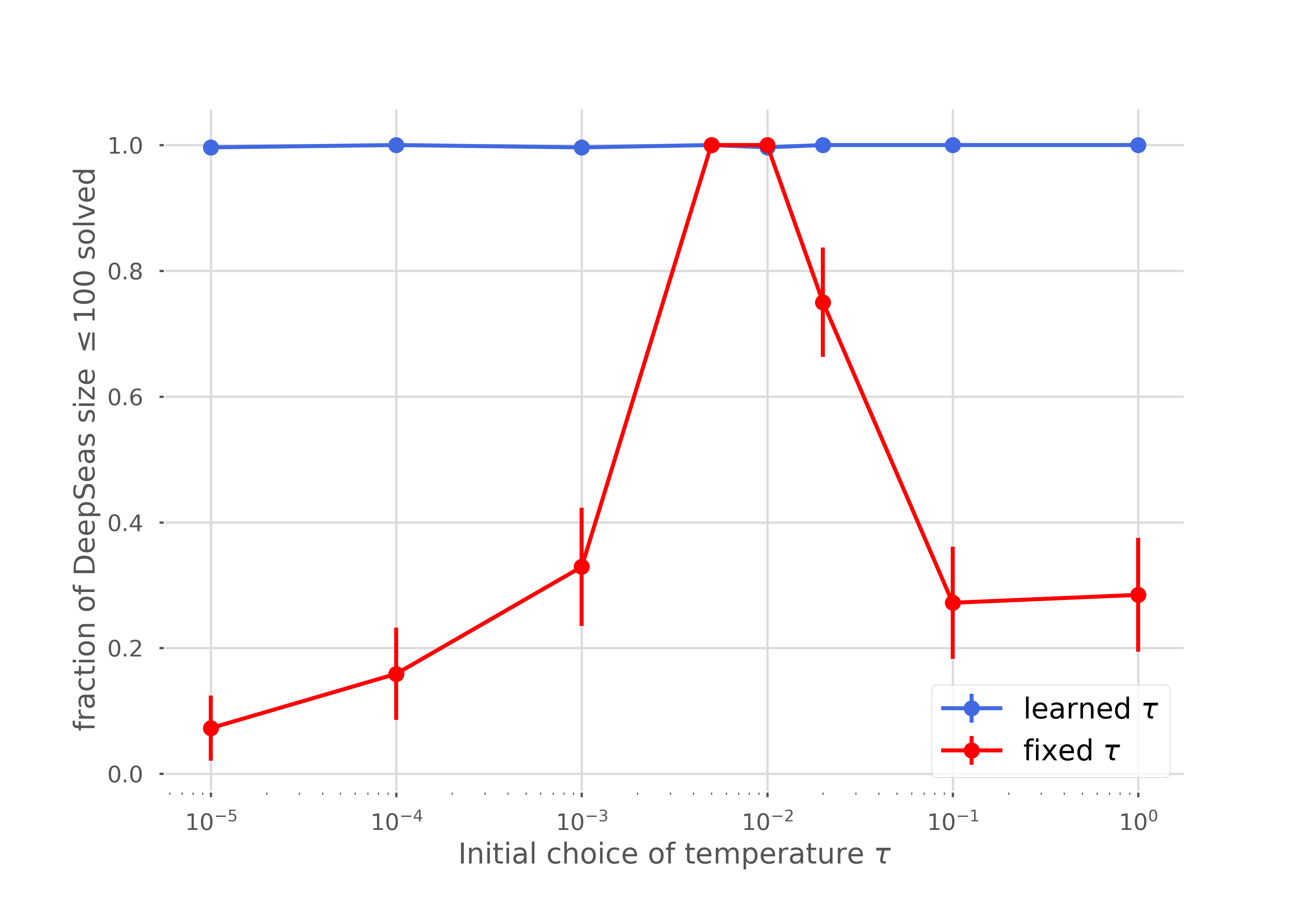}
  \caption{The optimal fixed risk-seeking parameter $\tau$ can produce good results, but learning $\tau$ via Algorithm \ref{a-kac} is far more robust.}
  \label{f-robustness_tau}
\end{center}
\end{figure}

The excellent performance of Algorithm~\ref{a-kac} on DeepSea raises the question: What is the maximum depth that the algorithm is able to consistently solve? We ran the algorithm on a DeepSea of depth 250 with $100$ random seeds to see how the performance degraded with depth. To handle the longer episode length before a reward we increased both the discount factor to $\gamma=0.999$ and the $\lambda$ factor in TD-$\lambda$ to $0.95$. The average performance is plotted in \ref{f-depth_250}. Overall, 99 out of the 100 seeds managed to reach the goal within $10^6$ episodes. In order to reach a positive reward the agent must make the exact right sequence of 250 actions and any deviation is impossible to recover from. This is a very difficult problem and one that would require an enormous number of episodes for a simple dithering agent, since $2^{250} \approx 10^{75}$. This suggests that Algorithm~\ref{a-kac} is able to handle extremely deep and difficult DeepSeas without much degradation in performance, and the limit has not yet been reached.

\begin{figure}[h]
\begin{center}
\includegraphics[scale=0.4]{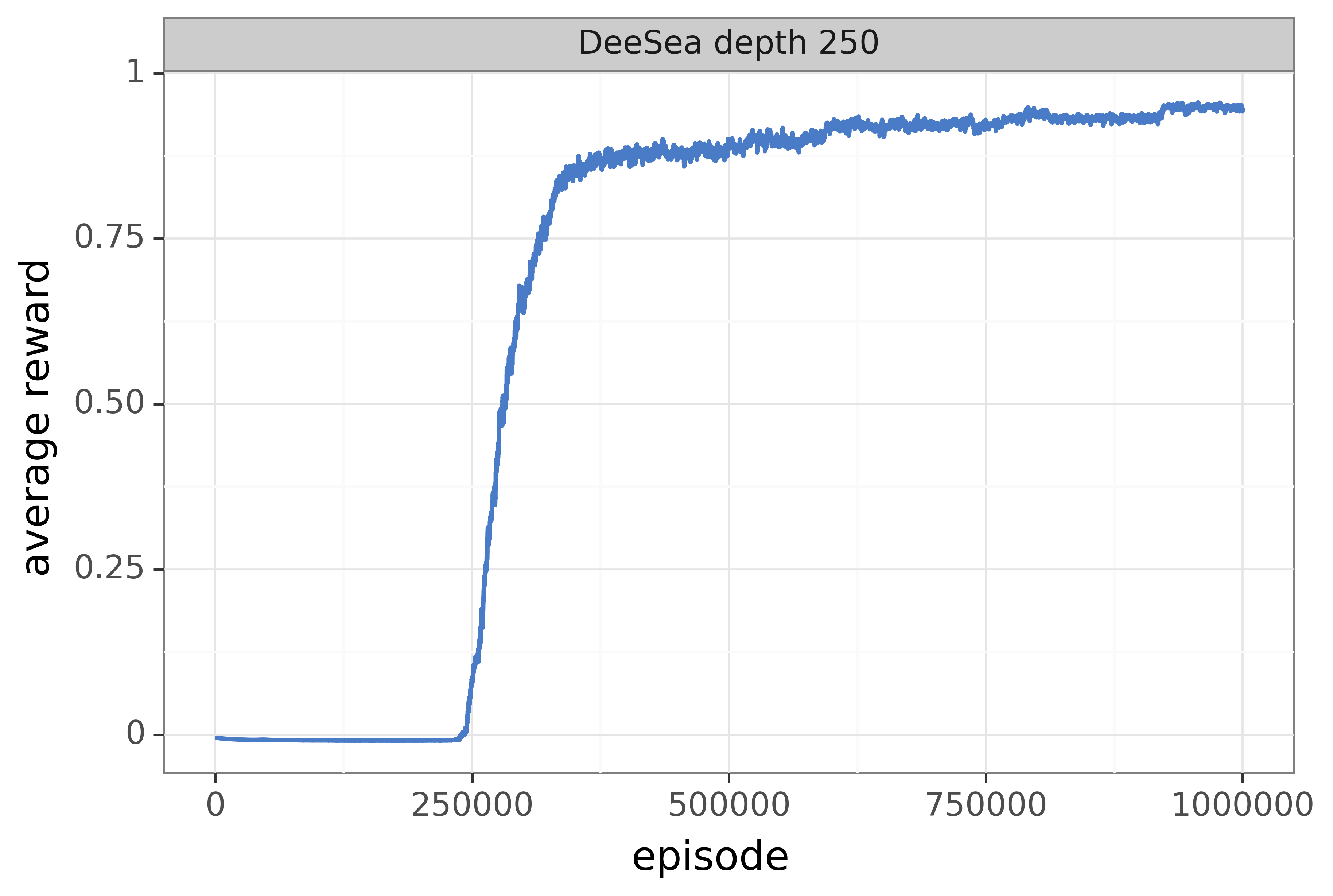}
  \caption{Performance on DeepSea of depth $250$ for Algorithm~\ref{a-kac} averaged over $100$ seeds. Overall, 99 out of 100 seeds managed to reach the goal within $10^6$ episodes.}
  \label{f-depth_250}
\end{center}
\end{figure}

\subsection{DeepSea replay experiments}

In Figure~\ref{f-replay_noise_v_no_noise} we show the benefit of adding noise to the reward samples in the replay buffer. It is clear that without the addition of noise the replay is destroying the uncertainty estimates and leading to worse performance than without replay. However, once the noise is added the agent with replay outperforms the purely online agent. Figure~\ref{f-replay_v_no_replay_not_log} is the same as Figure~\ref{f-replay_v_no_replay} except on a linear, rather than log, scale. Figure~\ref{f-solve_time_v_depth_kimpala} shows that adding replay to ERSAC does not appear to alter the quadratic dependency of solve time on depth.

\begin{figure}[!h]
\begin{center}
\includegraphics[scale=0.3]{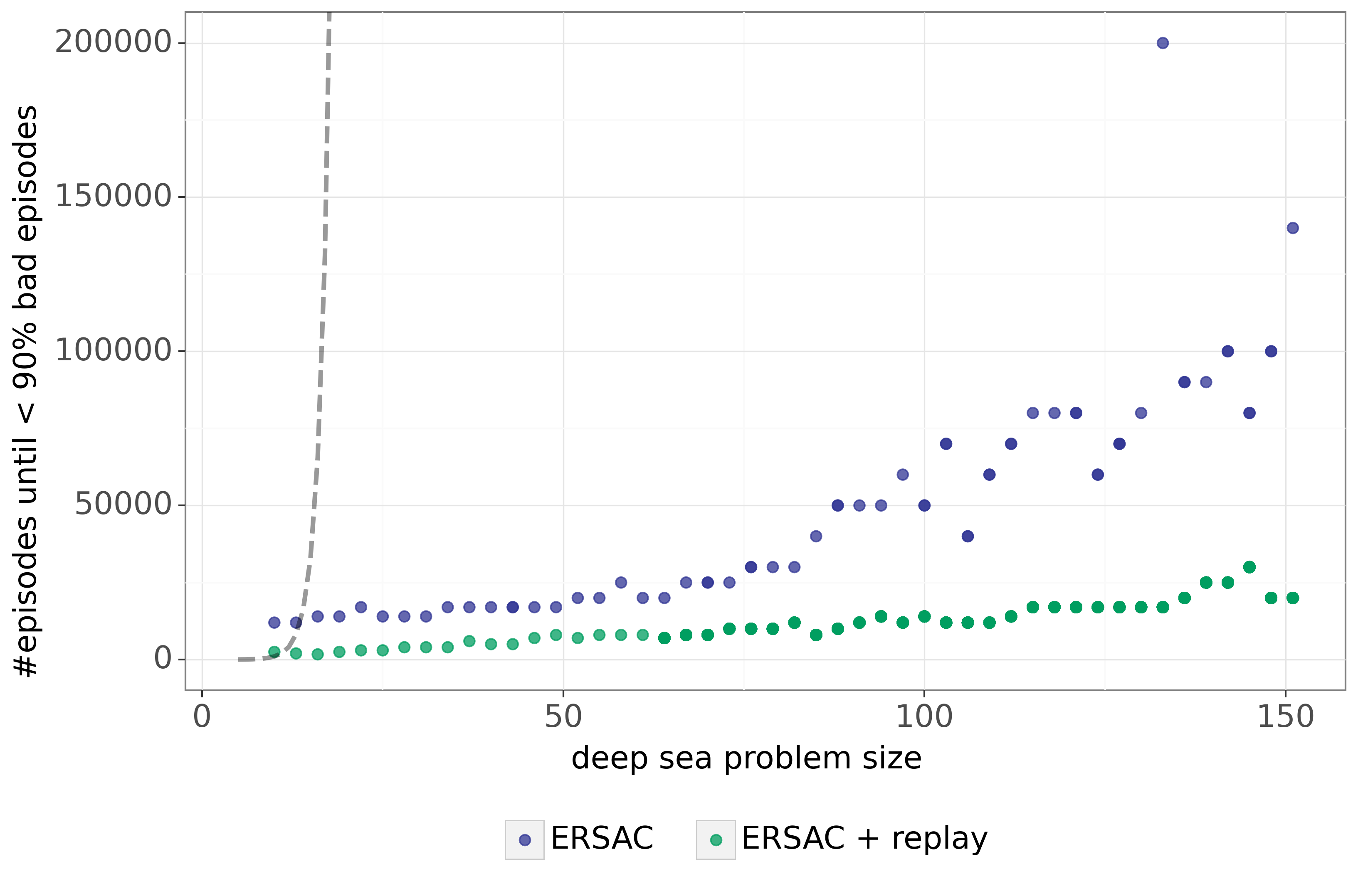}
  \caption{Adding replay to the epistemic-risk-seeking actor-critic improves data efficiency by a factor of about $4\times$. Note the depth here goes out to $151$.}
  \label{f-replay_v_no_replay_not_log}
\end{center}
\end{figure}

\begin{figure}[!h]
\begin{center}
\includegraphics[scale=0.3]{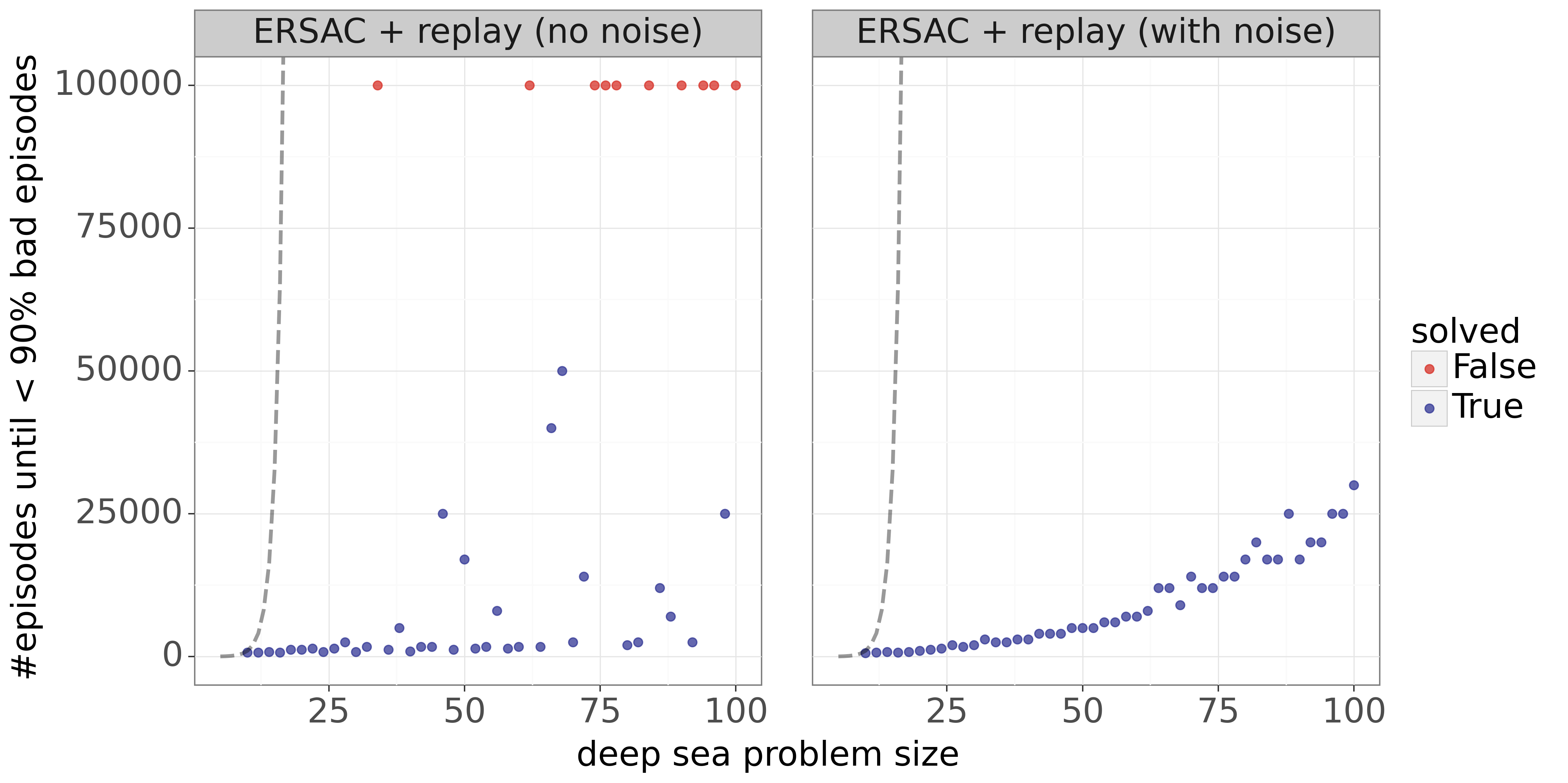}
  \caption{Adding noise to the reward targets when using replay dramatically improves the uncertainty estimates and the performance of the agent.\vspace{-4mm}}
  \label{f-replay_noise_v_no_noise}
\end{center}
\end{figure}

\begin{figure}[!h]
\begin{center}
\includegraphics[scale=0.3]{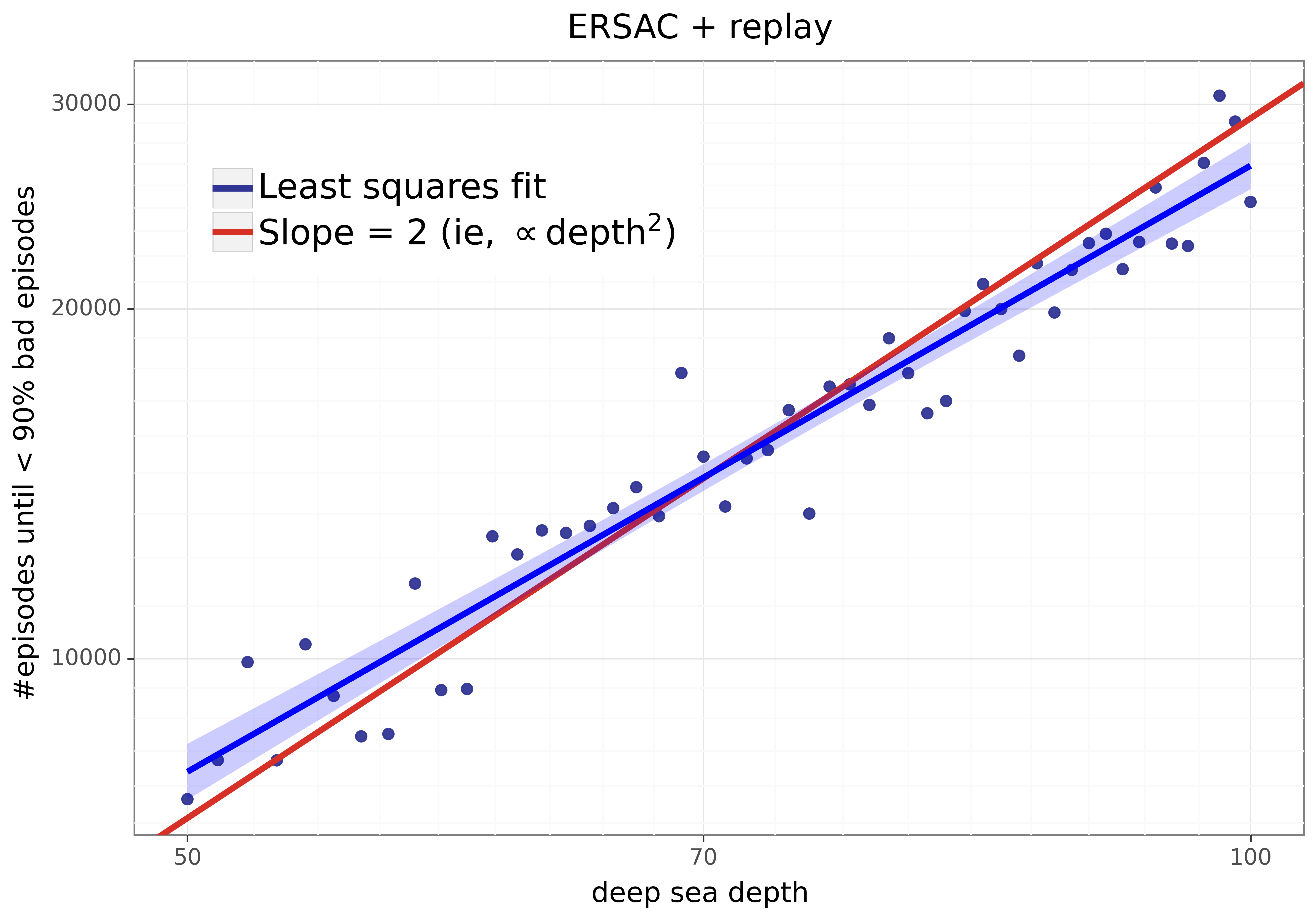}
  \caption{The solve time for ERSAC + replay on DeepSea has the same empirical quadratic dependency with depth as ERSAC without replay, but it is about $4\times$ faster overall.}
  \label{f-solve_time_v_depth_kimpala}
\end{center}
\end{figure}

\section{Atari results}
In Figure \ref{f-atari_hard_exp} we present the performance of Algorithm \ref{a-kac} compared to the vanilla actor-critic algorithm on a collection of $7$ hard exploration games from the Atari 57 suite \cite{bellemare-ale}. In Figure \ref{f-atari_all_games} we compare the performance of the agents across all 57 Atari games.

\label{app-atari}
\begin{figure*}[h]
\begin{center}
\includegraphics[scale=0.4]{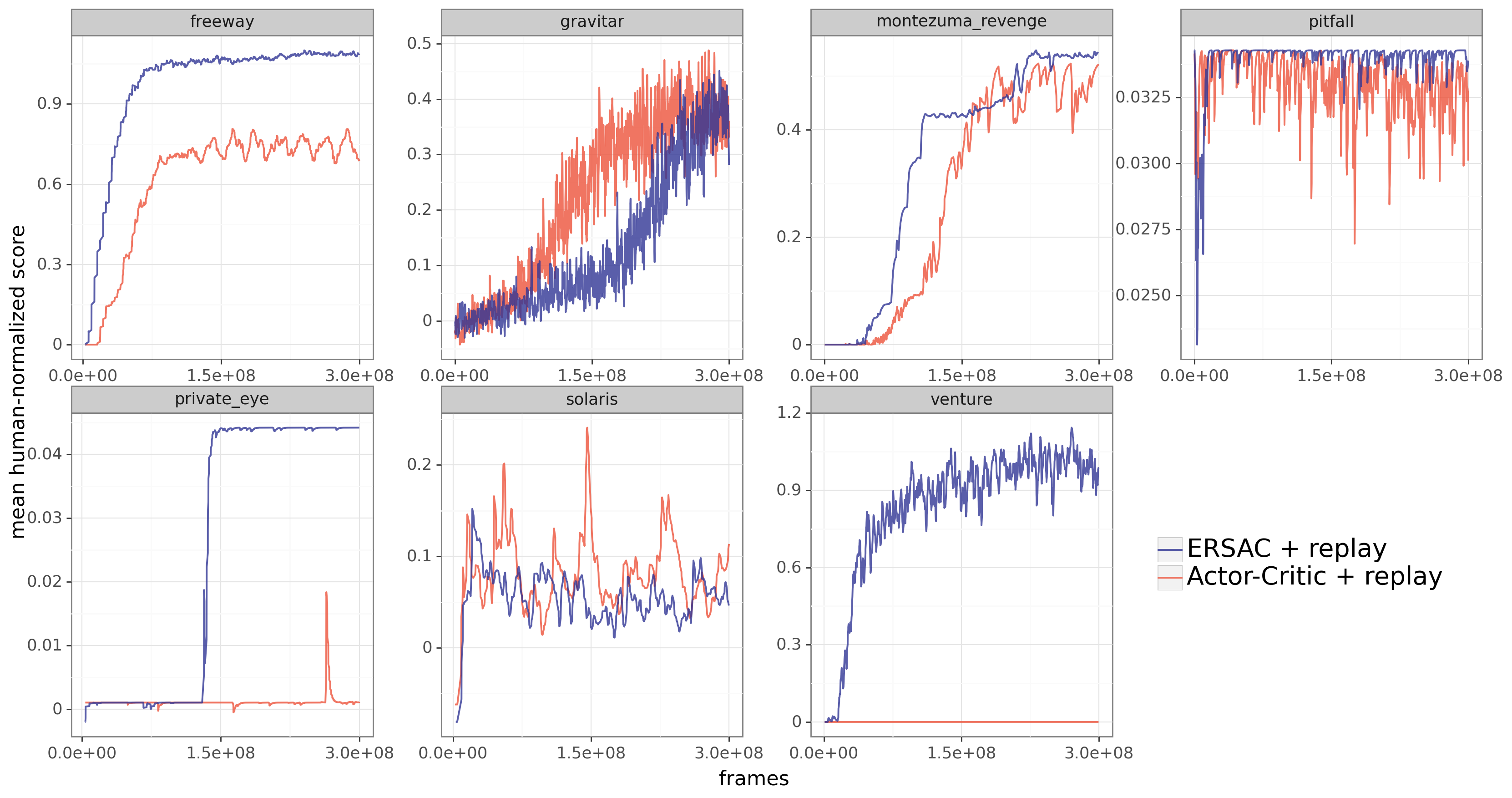}
  \caption{In this collection of hard exploration Atari games we see that the epistemic actor-critic algorithm provides a performance improvement over Replay actor-critic in four of the 7 games. In particular, there is a significant performance improvement for the very hard exploration game `Montezuma's revenge'.}
  \label{f-atari_hard_exp}
\end{center}
\end{figure*}

\begin{figure}[!hpb]
\begin{center}
\includegraphics[scale=0.27]{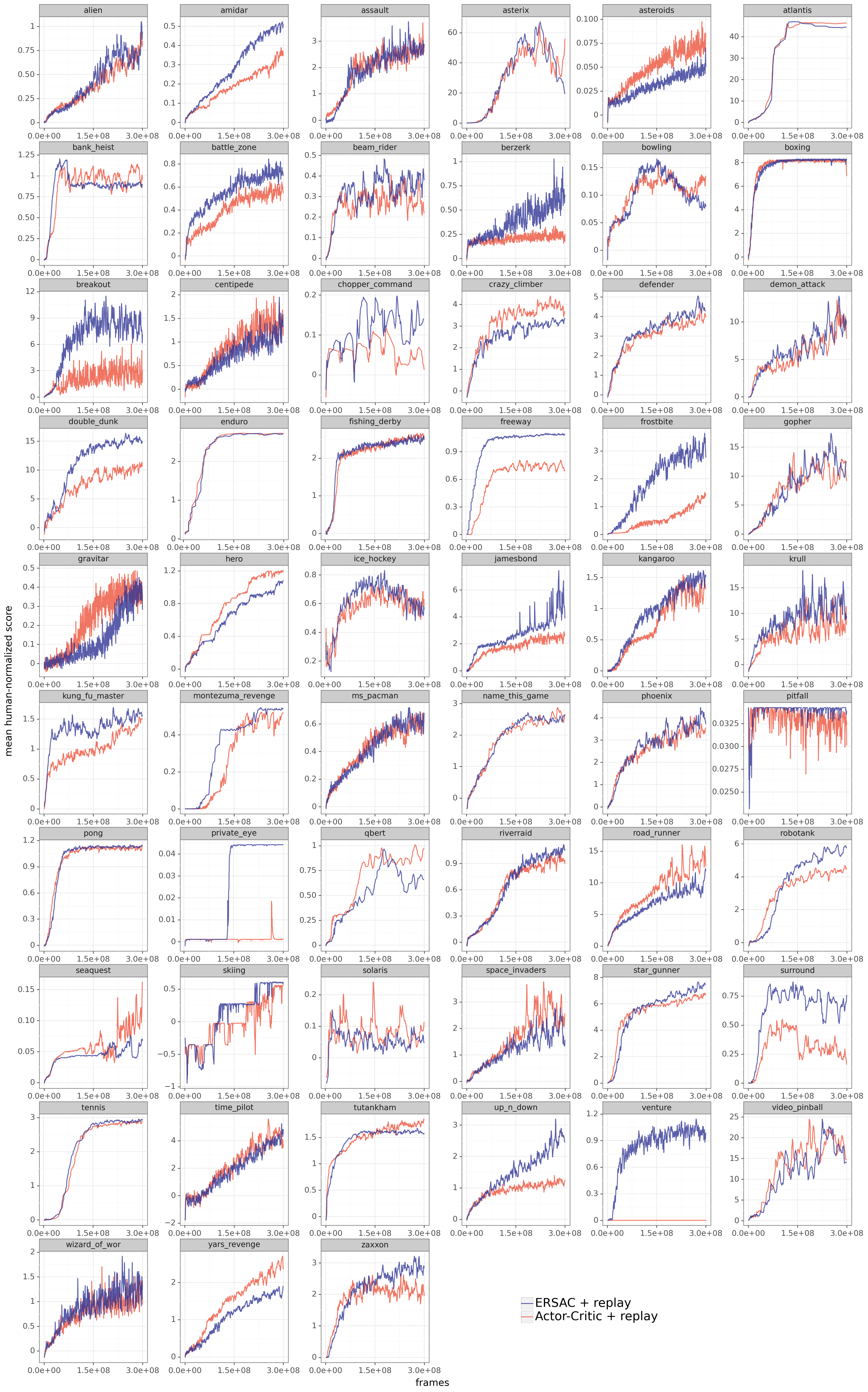}
    \caption{Performance of ERSAC and an actor-critic agent across all 57 Atari games.}
  \label{f-atari_all_games}
\end{center}
\end{figure}

\clearpage

\section{Future work}
\label{s-future-work}
We conclude with some discussion about future directions for this work. One question that this work raises is whether it is appropriate to have a single risk-seeking (entropy regularization) parameter $\tau$ for all states and actions \cite{ziebart2010modeling, neu2017unified, o2016pgq, nachum2017bridging, eysenbach2019if, haarnoja2018soft}. Some preliminary work \cite{o2021vbos} suggests that in fact it is both possible and advantageous to have a separate risk-seeking parameter for each state-action pair. In future work we may wish to investigate this. Simple actor-critic methods are no longer state-of-the-art, with most effective policy-based agents employing a range of different tactics to improve performance such as trust-regions, Q-value critics, natural gradients, model-based rollouts \etc~ An interesting extension would be to incorporate the techniques discussed here into those agents. We discussed at a high-level the regret of the formulation we derive in \S \ref{s-regret}, and showed empirical regret scaling results in Figure \ref{f-solve_time_v_depth}. In future work it would be interesting to combine the results of this work with theoretical results on the convergence rate of policy gradient algorithms to derive a concrete regret bound for a epistemic-risk-seeking policy-gradient algorithm.

\end{document}